\documentclass[twoside,11pt]{article}

\usepackage{jmlr2e}
\usepackage{subfigure}
\usepackage{booktabs}
\usepackage{multirow}
\usepackage{algorithm}
\usepackage{algorithmic}
\usepackage{amsmath}
\usepackage{cases}
\usepackage{amsfonts}

\newtheorem{assumption}{Assumption}

\usepackage{color}

\ShortHeadings{Adaptive Deep Learning for Entity Resolution by Risk Analysis}{Qun Chen et al.}
\firstpageno{1}

\begin{document}
	
	\title{Adaptive Deep Learning for Entity Resolution by Risk Analysis}
	
	\author{\name Zhaoqiang Chen \email chenzhaoqiang@mail.nwpu.edu.cn \\
		\name Qun Chen \email chenbenben@nwpu.edu.cn \\
		\name Youcef Nafa \email youcef.nafa@mail.nwpu.edu.cn \\
		\name Tianyi Duan \email tianyiduan@mail.nwpu.edu.cn \\
		\name Wei Pan \email panwei1002@nwpu.edu.cn \\
		\name Lijun Zhang \email zhanglijun@nwpu.edu.cn \\
		\name Zhanhuai Li \email lizhh@nwpu.edu.cn \\
		\addr School of Computer Science, Northwestern Polytechnical University\\
		Key Laboratory of Big Data Storage and Management, Northwestern Polytechnical University, Ministry of Industry and Information Technology\\
		Xi'an, China}
	
	\editor{Anonymous}
	
	\maketitle
	
	\begin{abstract}
		The state-of-the-art performance on entity resolution (ER) has been achieved by deep learning. However, deep models are usually trained on large quantities of accurately labeled training data, and can not be easily tuned towards a target workload. Unfortunately, in real scenarios, there may not be sufficient labeled training data, and even worse, their distribution is usually more or less different from the target workload even when they come from the same domain. 
		
		To alleviate the said limitations, this paper proposes a novel risk-based approach to tune a deep model towards a target workload by its particular characteristics. Built on the recent advances on risk analysis for ER, the proposed approach first trains a deep model on labeled training data, and then fine-tunes it by minimizing its estimated misprediction risk on unlabeled target data. Our theoretical analysis shows that risk-based adaptive training can correct the label status of a mispredicted instance with a fairly good chance. 
		We have also empirically validated the efficacy of the proposed approach on real benchmark data by a comparative study. Our extensive experiments show that it can considerably improve the performance of deep models. Furthermore, in the scenario of distribution misalignment, it can similarly outperform the state-of-the-art alternative of transfer learning by considerable margins. Using ER as a test case, we demonstrate that risk-based adaptive training is a promising approach potentially applicable to various challenging classification tasks. 
	\end{abstract}
	
	
	\section{Introduction}
\label{sec:intro}

  Extensively studied in the literature~\citep{christen2008automatic}, ER is an important problem for data integration. Its purpose is to identify the equivalent records that refer to the same real-world entity.  Considering the running example shown in Figure~\ref{fig:runningexample}, ER needs to match the paper records between two tables, $R_1$ and $R_2$. A pair of $<r_{1i},r_{2j}>$, in which $r_{1i}$ and $r_{2j}$ denote a record in $R_1$ and $R_2$ respectively, is called an \emph{equivalent} pair if and only if $r_{1i}$ and $r_{2j}$ refer to the same paper; otherwise, it is called an \emph{inequivalent} pair. In this example, $r_{11}$ and $r_{21}$ are \emph{equivalent} while $r_{11}$ and $r_{22}$ are \emph{inequivalent}. ER can be considered as a binary classification problem tasked with labeling record pairs as \emph{matching} or \emph{unmatching}. Therefore, various learning models have been proposed for ER~\citep{christen2008automatic}. As many other classification tasks (e.g. image and speech recognition), the state-of-the-art performance on ER has been achieved by deep learning~\citep{joty2018distributed,mudgal2018deep,nie2019deep,zhao2019auto,li2020deep}.

\begin{figure}
	\begin{center}
		\centerline{\includegraphics[width=\linewidth]{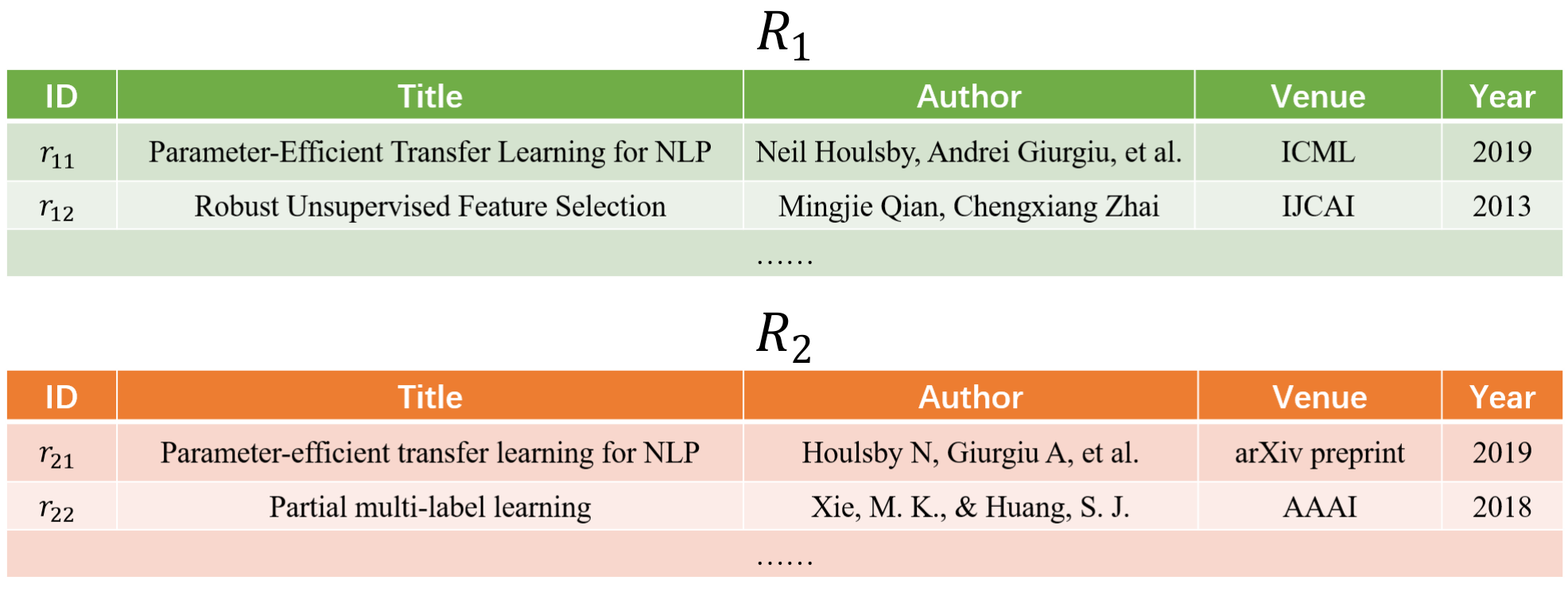}}
		\caption{An ER running example.}
	\end{center}
	\label{fig:runningexample}
\end{figure}

\begin{figure}
	\begin{center}
		\centerline{\includegraphics[width=\linewidth]{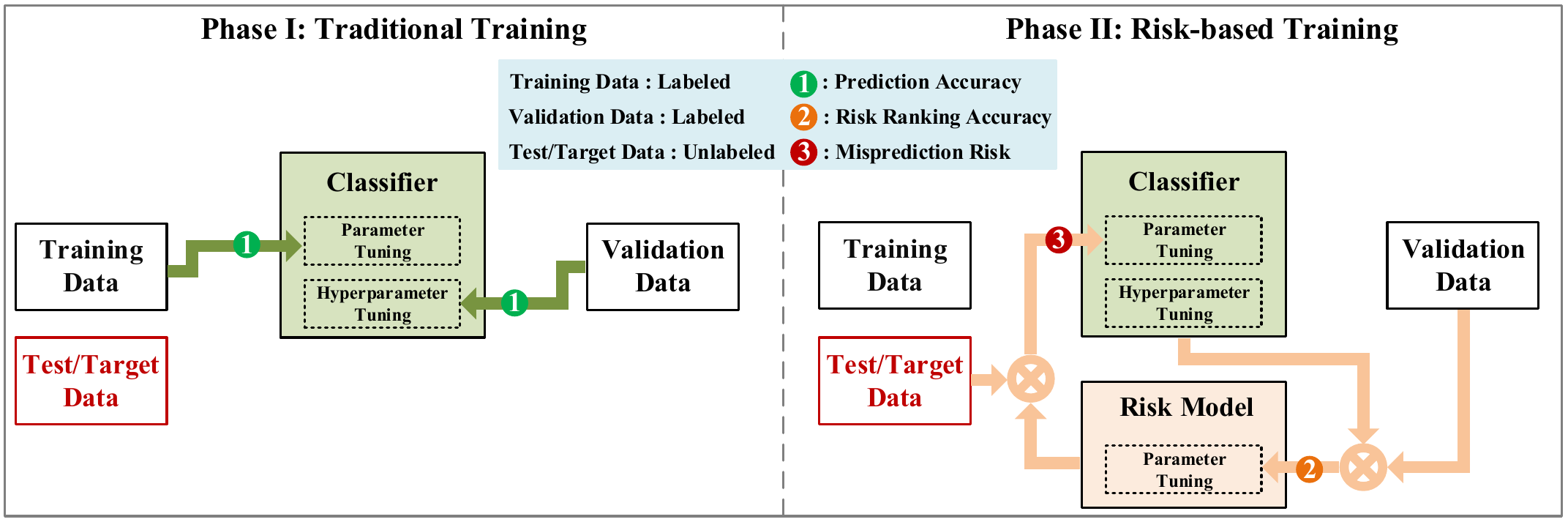}}
		\caption{Risk-based Adaptive Training.}
	\end{center}
	\label{fig:hybrid-train}
\end{figure}

However, the efficacy of Deep Neural Network (DNN) models depends on large quantities of accurately labeled training data, which may not be readily available in real scenarios. Furthermore, in the typical setting of deep learning, model parameters are tuned on \emph{labeled training data} in a way ensuring that the resulting classifier's predictions on the training instances are most consistent with their ground-truth labels. The trained classifier is then applied on a target workload. It can be observed that the typical process of model training does not involve the unlabeled data in a target workload, even though to alleviate the over-fitting problem, labeled validation data are usually provided as the proxy workload of a target task and leveraged for hyperparameter tuning. Theoretically, the efficacy of this approach is based on the assumption that training and target data are {\bf i}ndependently and {\bf i}dentically {\bf d}istributed (the {\bf i.i.d} assumption). Unfortunately, in real scenarios, even when training and target data come from the same domain, the aforementioned assumption may not hold due to: 1) Training data are not sufficient to fully represent the statistical characteristics of a target workload; 2) Even though training data are abundant, its inherent distribution may be to some extent different from the target workload. Therefore, it is not unusual in real scenarios that a well trained deep model performs unsatisfactorily on a target workload.

Many adaptation approaches have been proposed to alleviate distribution misalignment, most notably among them \emph{transfer learning}~\citep{DBLP:journals/tkde/PanY10,DBLP:conf/icml/WeiZHY18,DBLP:conf/icml/HoulsbyGJMLGAG19} and \emph{adaptive representation learning}~\citep{DBLP:conf/cvpr/LongD0SGY13,DBLP:conf/icml/LongC0J15,DBLP:conf/icml/0002CZG19,wu2019ace,kim2019diversify}. Transfer learning aimed to adapt a model learned on the training data in a source domain to a target domain. Similarly, adaptive representation learning, which was originally proposed for image classification, mainly studied how to learn domain-invariant features shared among diversified domains. Unfortunately, distribution misalignment remains very challenging. One of the reasons is that the existing approaches focused on how to extract and leverage the common knowledge shared between a source task and a target task; however, they have limited capability to tune deep models towards a target task by its particular characteristics.

It has been well recognized that in real scenarios, with or without adaptation, a well-trained classifier may not be accurate in predicting the labels of the instances in a target workload. Even worse, it may provide high-confidence predictions which turn out to be wrong~\citep{DBLP:journals/corr/GoodfellowSS14}. Such prediction uncertainty has emerged as a critical concern to AI safety~\citep{amodei2016concrete}. Therefore, various techniques~\citep{hendrycks2016baseline,chen2018improving,jiang2018trust,hendrycks2019deep,chen2019towards} have been proposed for the task of risk analysis, which aims at estimating the misprediction risk of a classifier when applied to a certain workload.

Since risk analysis can measure the misprediction risk of a classifier on unlabeled data, it provides classifier training with a viable way to adapt towards a particular workload. Hence, we propose a risk-based approach to enable adaptive deep learning for ER in this paper.
It is noteworthy that the recently proposed LearnRisk~\citep{chen2019towards} is an interpretable and learnable risk analysis approach for ER. Compared with the simpler alternatives, LearnRisk is more interpretable and can identify mislabeled pairs with considerably higher accuracy. Therefore, we build the solution of adaptive deep training upon LearnRisk in this paper. 

 The proposed approach is shown in Figure 2, in which \emph{test data} represent a target workload. 
It consists of two phases, the phase of \emph{traditional training} followed by the phase of \emph{risk-based training}. In the first phase, a deep model is trained on labeled training data in the traditional way. In the second phase, it is further tuned to minimize the misprediction risk on unlabeled target data. 
The main contributions of this paper are as follows:
\begin{itemize}
	\item We propose a novel risk-based approach to enable adaptive deep learning. 
	\item We present a solution of adaptive deep learning for ER based on the proposed approach.
	\item We theoretically analyze the performance of the proposed solution for ER. Our analysis shows that risk-based adaptive training can correct the label status of a mispredicted instance with a fairly good chance.
	\item We empirically validate the efficacy of the proposed approach on real benchmark data by a comparative study.
\end{itemize}

The rest of this paper is organized as follows: Section~\ref{sec:related} reviews related work. Section~\ref{sec:preliminaries} presents the preliminaries. Section~\ref{sec:risktrain} presents the adaptive training approach and the theoretical results. Section~\ref{sec:experiments} describes the empirical evaluation results. Finally, Section~\ref{sec:conclusion} concludes this paper.

	\section{Related Work}
\label{sec:related}

We review related work from three mutually orthogonal perspectives: entity resolution, model training and ensemble learning.

\vspace{0.1in}
\hspace{-0.25in}\textbf{Entity Resolution}. 
ER plays a key role in data integration and has been extensively studied in the literature~\citep{christen2012data,christophides2015entity,elmagarmid2007duplicate}. ER can be automatically performed based on rules~\citep{fan2009reasoning,li2015rule,singh2017generating}, probabilistic theory~\citep{fellegi1969theory,singla2006entity} and machine learning models~\citep{christen2008automatic,cochinwala2001efficient,kouki2017collective,sarawagi2002interactive}. The state-of-the-art performance on ER has been achieved by deep learning~\citep{joty2018distributed,mudgal2018deep,nie2019deep,zhao2019auto,li2020deep}. Specifically, a deep learning architecture template for ER has been provided in \citep{mudgal2018deep}. More recently, in the following work, \citet{li2020deep} presented an improved solution based on pre-trained Transformer-based language models, e.g., BERT. It is noteworthy that this paper does not attempt to propose a new DNN model for ER. It instead focuses on how to tune deep models towards a target task via risk analysis. Therefore, the existing work on deep learning for ER is orthogonal to ours. In principle, our proposed approach can work with any DNN for ER.

ER remains very challenging in real scenarios due to prevalence of dirty data. Therefore, there is a need for \emph{risk analysis}, alternatively called \emph{trust scoring} and \emph{confidence ranking} in the literature. The proposed solutions ranged from those simply based on the model's output probabilities to more sophisticated interpretable and learnable ones \citep{zhang2014predicting,hendrycks2016baseline,jiang2018trust,chen2019towards}. Most recently,  we proposed an interpretable and learnable framework for ER, LearnRisk~\citet{chen2019towards}, which is
able to construct a dynamic risk model tuned towards a specific workload. In our following work ~\citet{NAFA2022107729}, we also proposed to actively select training data for ER models based on the results of risk analysis. In this paper, we leverage the results of risk analysis to fine-tune deep models
toward a specific workload.

\vspace{0.1in}
\hspace{-0.25in}\textbf{Model Training}.
A common challenge for model training is the \emph{over-fitting}, which refers to the phenomenon that a model well tuned on training data performs unsatisfactorily on target data. 
The de-facto standard approach to alleviate \emph{over-fitting} is by leveraging validation data for hyperparameter tuning and model selection (e.g. cross validation)~\citep{kohavi1995study}. 
Another noteworthy complementary technique is the \emph{regularization}~\citep{evgeniou2000regularization,DBLP:conf/nips/BaldiS13,neyshabur2017exploring,DBLP:conf/iclr/ZhangBHRV17}, which aims to reduce the number of model parameters to a manageable level. Both hyperparameter tuning and model selection are to a large extent orthogonal to model training considered in this paper. These works are therefore orthogonal to our work. 

   The classical way to alleviate the insufficiency of labeled training data, semi-supervised learning, has been extensively studied in the literature~\citep{zhu2005semi,DBLP:conf/cvpr/IscenTAC19}. However, semi-supervised learning investigated how to leverage unlabeled training data, which usually have a similar distribution with labeled training data. It is obvious that the techniques for semi-supervised learning can be straightforwardly incorporated into the pre-training phase of our proposed approach. They are therefore orthogonal to our work. Another way to reduce labeling cost is by active learning~\citep{settles2012active,ducoffe2018adversarial}. While active learning focused on how to select training data for labeling,  we focused on how to adapt a model towards a target workload provided with a set of training data. Therefore, active learning is also orthogonal to our work. 

\vspace{0.1in}
\hspace{-0.25in}\textbf{Ensemble Learning.}
To address the limitations of a single classifier, ensemble learning has been extensively studied in the literature~\citep{zhou2009ensemble,sagi2018ensemble}. Ensemble learning first trains multiple classifiers using different training data (e.g., bagging~\citep{breiman1996bagging}) or different information in the same training data (e.g., boosting~\citep{schapire1990strength}), and then combine probably conflicting predictions to arrive at a final decision. While LearnRisk uses the ensemble of risk features to measure misprediction risk, risk-based adaptive training is fundamentally different from ensemble learning due to the following two reasons: 1) unlike the traditional labeling functions, \emph{LearnRisk} aims to estimate an instance's misclassification risk as predicted by a classifier; 2) more importantly, the ensemble approach trains multiple models and tunes predictions based on training data; in contrast, risk-based adaptive training trains only one model, and tunes the model towards a target workload. On the other hand, since ensemble learning trains models based on training data, the work on ensemble learning is orthogonal to ours. In principle, our proposed approach can also work with an ensemble learning model. However, how to tune an ensemble model via risk measure however requires further investigation.

	\section{Preliminaries}
\label{sec:preliminaries}
In this section, we first define the ER classification task and then introduce the state-of-the-art risk analysis technique for ER, LearnRisk.

\subsection{Task Statement}
\label{sec:task}
This paper considers ER as a binary classification problem. A classifier needs to label every unlabeled pair as {\em matching} or {\em unmatching}. As usual, we measure the quality of an ER solution by the standard metric of \emph{F1}, which is a combination of \emph{precision} and \emph{recall} as follows
\begin{equation}
 F1=\frac{2\times precision\times recall}{precision + recall}.
\end{equation}

\begin{table}[ht]
	\caption{The Frequently Used Notations}
	\label{tab:notations}
	\begin{center}
		\begin{tabular}{|l|l|}
			\toprule
			Notation & Description \\ \hline
			\midrule
			$D$ & an ER workload \\ \hline
			$D^s, D^v, D^t$ & subsets of $D$, corresponding to training set, validation set and test set \\\hline
			$d_i$ & an instance pair in $D$ \\ \hline
			$\mathbf{x}_i$ & the feature representation vector of $d_i$ \\ \hline
			$y_i$ & the label of $d_i$ \\ \hline
			$\mu_{d_i}$ & the expectation of equivalence probability of $d_i$ \\ \hline
			$\sigma_{d_i} (resp. \sigma^2_{d_i})$ & the standard deviation (resp. variance) of equivalence probability of $d_i$ \\ \hline
			$f_i$ & a risk feature \\ \hline
			$w_i$ & the feature weight of $f_i$  \\
			\bottomrule
		\end{tabular}
	\end{center}
\end{table}

For presentation simplicity, we summarize the frequently used notations in Table~\ref{tab:notations}. As usual, we suppose that an ER task $D$ consists a set of labeled training data, $D^s=\{(\mathbf{x}^s_i, y^s_i)|i\}$, where each $(\mathbf{x}^s_i, y^s_i)$ denotes a training instance with its feature representation $\mathbf{x}^s_i$ and ground-truth label $y^s_i$, a set of labeled validation data, $D^v=\{(\mathbf{x}^v_i, y^v_i)|i\}$, and a set of unlabeled test data, $D^t=\{(\mathbf{x}^t_i, ?)|i\}$. Note that $D^t$ denotes the target workload, and $D^v$ is provided as a proxy workload of $D^t$. Formally, we define the classification task of ER as
 
\begin{definition}
 \textbf{[ER Classification Task]}. Given an ER workload $D$ consisting of $D^s$, $D^v$ and $D^t$, the task aims to learn an optimal classifier, $C_*$, based on $D$ such that the performance of $C_*$ on $D^t$ as measured by the metric of $F1$, or $F1(C_*,D^t)$, is maximized. 
\end{definition}

\subsection{Risk Analysis for ER: LearnRisk}
\label{sec:learnrisk}

The risk analysis pipeline of LearnRisk operates in three main steps: \textit{risk feature generation} followed by \textit{risk model construction} and finally \textit{risk model training}.

\subsubsection{Risk feature generation}
The step automatically generates risk features in the form of interpretable rules based on one-sided decision trees. The algorithm ensures that the resulting rule-set is both discriminative, i.e, each rule is highly indicative of one class label over the other; and has a high data coverage, i.e, its validity spans over a subpopulation of the workload. As opposed to classical settings where a rule is used to label pairs to be equivalent or inequivalent, a risk rule feature focuses exclusively on one single class. Consequently, risk features act as indicators of the cases where a classifier's prediction goes against the knowledge embedded in them. An example of such rules is:
\begin{equation} \label{eq:rule-example}
  r_i[Year]\neq r_j[Year] \rightarrow inequivalent(r_{i}, r_{j}),
\end{equation}	
where $r_i$ denotes a record and $r_i[Year]$ denotes $r_i$'s attribute value at $Year$. 
With this knowledge, a pair predicted as \emph{matching} whose two records have different publication years is assumed to have a high risk of being mislabeled.

\subsubsection{Risk model construction}
Once high-quality features have been generated, the latter are readily available for the risk model to make use of, allowing it to be able to judge a classifier's outputs backing up its decisions with human-friendly explanations. To achieve this goal, LearnRisk, drawing inspiration from investment theory, models each pair's equivalence probability distribution (portfolio reward) as the aggregation of the distributions of its compositional features (stocks rewards). 

Practically, the equivalence probability of a pair $d_i$ is modeled by a random variable $p_i$ that follows a normal distribution $\mathcal{N}(\mu_i, \sigma_i^2)$, where $\mu_i$ and $\sigma_i^2$ denote expectation and variance respectively. Given a set of $m$ risk features ${f_1,f_2,...,f_m}$, let ${w_1,w_2,...,w_m}$ denote their corresponding weights. Suppose that $\mathbf{\mu}_F=$ $[\mu_{f_1},$ $\mu_{f_2},$ $\ldots, \mu_{f_m}]^T$ and $\mathbf{\sigma}^2_F=$ $[\sigma_{f_1}^2,$ $\sigma_{f_2}^2,$ $\ldots, \sigma_{f_m}^2]^T$ represent their corresponding expectation and variance vectors respectively, such that $\mathcal{N}(\mu_{f_j},\sigma_{f_j}^2)$ denotes the equivalence probability distribution of the feature $f_j$. Accordingly, the distribution parameters for $d_i$ are estimated by: 
\begin{equation}
  \mu_i = \mathbf{z}_i (\mathbf{w} \circ \mathbf{\mu}_F), 
\end{equation}
and
\begin{equation}
  \sigma_i^2 = \mathbf{z}_i (\mathbf{w} \circ \mathbf{w} \circ \mathbf{\sigma}^2_F), 
\end{equation}
Where $\circ$ represents the element-wise product and $\mathbf{z}_i$ is a one-hot feature vector.

\emph{Note that besides one-sided decision rules, LearnRisk also incorporates classifier output as one of the risk features, which is referred to as the DNN risk feature in this paper.} Provided with the equivalence distribution $p_i$ for $d_i$, LearnRisk measures its risk by the metric of Value-at-Risk (VaR)~\citep{tardivo2002value}, which can effectively capture the fluctuation risk of label status. Provided with a confidence level of $\theta$, the metric of VaR represents the maximum loss after excluding all worse outcomes whose combined probability is at most 1-$\theta$.

\subsubsection{Risk model training}
Finally, the risk model is trained on labeled validation data to optimize a learn-to-rank objective by tuning the risk feature weight parameters ($w_i$) as well as their variances ($\sigma_i^2$). As for their expectations ($\mu_i$), they are considered as prior knowledge, and estimated from labeled training data. Once trained, the risk model can be used to assess the misclassification risk on an unseen workload labeled by a classifier.

	\section{Risk-based Adaptive Training}
\label{sec:risktrain}

In this section, we propose, and then analyze the approach of risk-based adaptive training for ER. We take \emph{DeepMatcher}~\citep{mudgal2018deep}, an classical DNN solution for ER, as an example to illustrate the solution. However, it is worthy to point out that 
in principle, the proposed approach can work with any DNN classifier; the implementation of the proposed solution on other DNNs is similar. In our empirical evaluation, we have implemented the proposed solution on both \emph{deepmatcher} and \emph{Ditto}~\citep{li2020deep}, the most recent DNN for ER based on Transformer-based language models. For comparison, we also briefly describe the traditional training approach.

\subsection{Traditional Training} \label{sec:traditional-train}

Given a workload consisting of $D^s$, $D^v$ and $D^t$, let $g(\mathbf{\omega})$ denote a DNN classifier with the parameters of $\mathbf{\omega}$. The traditional approach, as shown in the left part of Figure 2, tunes $\mathbf{\omega}$ towards the training data, $D^s$, based on a pre-specified loss function. Suppose that there are totally $n_s$ training instances in $D^s$. DeepMatcher employs the classical cross-entropy loss function to guide the process of parameter optimization as follows
\begin{equation}
\begin{split}
	\mathcal{L}_{train}(\mathbf{\omega}) = \frac{1}{n_s}\sum_{i=1}^{n_s}\{&-y^s_ilog(g(\mathbf{x}^s_i;\mathbf{\omega}))- \\
		&(1-y^s_i)log(1-g(\mathbf{x}^s_i;\mathbf{\omega}))\},
\end{split}
\end{equation}
where $y^s_i$ denotes the ground-truth label of a training instance, $(\mathbf{x}^s_i, y^s_i)$, and $g(\mathbf{x}^s_i;\mathbf{\omega})$ denotes its label probability as predicted by the classifier. DeepMatcher uses the Adam optimizer~\citep{kingma2014adam} to search for the optimal parameters $\mathbf{\omega}_*$ by gradient descent. 

\subsection{Risk-based Adaptive Training} \label{sec:hybrid-train}
\label{sec:framework}

Risk-based adaptive training, as shown in Figure 2, consists of two phases, a \emph{traditional training} phase followed by an \emph{risk-based training} phase. In the first phase, it tunes a deep model towards training data in the traditional way. Then, in the following \emph{risk-based training} phase, it iteratively performs: i) using LearnRisk to learn a risk model based on a trained classifier and validation data; ii) fine-tuning the classifier by minimizing its misprediction risk upon the target workload.

 Specifically, the loss function of risk-based training is defined as 
\begin{equation} \label{eq:riskloss}
\begin{split}
	\mathcal{L}_{test}^{risk}(\mathbf{\omega}) = \frac{1}{n_t}\sum_{i=1}^{n_t}\{&-[1- VaR^+(d_i)]log( g(\mathbf{x}^t_i;\mathbf{\omega})) - \\
	&[1-VaR^-(d_i)]log(1 - g(\mathbf{x}^t_i;\mathbf{\omega}))\},
\end{split}
\end{equation}
in which $n_t$ denotes the total number of instances in $D^t$, $VaR^+(d_i)$ (resp. $VaR^-(d_i)$) denotes the estimated misprediction risk of $d_i$ if it is labeled as \emph{matching} (resp. \emph{unmatching}).

Similar to traditional training, the \emph{risk-based training} phase updates the parameters of the deep model by gradient descent as follows
\begin{equation}
	\omega_{k+1} = \omega_{k} - \alpha * \nabla_{\omega_k}\mathcal{L}_{test}^{risk}(\omega).
\end{equation}
Note that in each iteration, risk values are estimated based on the classifier predictions of the previous iteration. As a result, they are considered as constant while computing gradient descent .

We have sketched the process of risk-based adaptive training in Algorithm~\ref{alg:risktrain}. The first phase pre-trains a model based on labeled training data, and selects the best one based on its performance on the validation data. Beginning with the pre-trained model, the second phase iteratively fine-tunes the parameters by minimizing the loss of $\mathcal{L}_{test}^{risk}(\mathbf{\omega})$.

\begin{algorithm}[t]
	\caption{Risk-based Adaptive Training}
	\label{alg:risktrain}
	\begin{algorithmic}
		\STATE {\bfseries Input:} A task $D$ consisting of $D^s$, $D^v$ and $D^t$, and an ER model, $g(\mathbf{\omega})$; 
		\STATE {\bfseries Output:} A learned classifier $g(\mathbf{\omega}_*)$.
		\STATE $\mathbf{\omega}_0 \gets$Initialize $\mathbf{\omega}$ with random values;
		\FOR{$k=0$ {\bfseries to} $m-1$}
		\STATE $\mathbf{\omega}_{k+1} \gets \mathbf{\omega}_{k} - \alpha* \nabla_{\mathbf{\omega}_k}\mathcal{L}_{train}(\mathbf{\omega}_k)$;
		\ENDFOR
		\STATE  Select the best model, $g(\mathbf{\omega}_*)$, based on $D^v$;
		\STATE $\mathbf{\omega}_m \gets \mathbf{\omega}_*$;
		\FOR{$k=m$ {\bfseries to} $m+n-1$}
		\STATE Update the risk model based on $D^v$ and $g(\mathbf{\omega}_k)$;
		\STATE $\mathbf{\omega}_{k+1} \gets \mathbf{\omega}_{k} - \alpha* \nabla_{\mathbf{\omega}_k}\mathcal{L}_{test}^{risk}(\mathbf{\omega}_k)$;
		\ENDFOR
		\STATE Select the best model, $g(\mathbf{\omega}_*)$, based on $D^v$.
		\STATE {\bfseries Return} $g(\mathbf{\omega}_*)$
	\end{algorithmic}
\end{algorithm}

\subsection{Theoretical Analysis}

Empirically, it is widely observed that DNNs are highly expressive leading to very low training errors provided with correct information. It can also be observed that given an estimated distribution of equivalence probability, ($\mu_i$,$\sigma^2_i$), for a pair $d_i$, $VaR^+(d_i) < VaR^-(d_i)$ if and only if $\mu_i > 0.5$. Hence, according the loss function defined in Eq.\ref{eq:riskloss}, for an equivalent pair of $d_i$ in $D^t$, $\mu_i > 0.5$ would result in it being correctly classified. Similarly, if $d_i$ is an inequivalent pair, $\mu_i < 0.5$ would result in it being correctly classified.

Suppose that \emph{LearnRisk} generates totally $m$ risk features, denoted by \{$f_1$,$\dots$,$f_m$\}. Let $Z_i$ be a 0-1 variable indicating whether an instance has the risk feature $f_i$: $Z_i=1$ if the instance has $f_i$, otherwise $Z_i=0$. 
Let $\mathbf{Z}=(Z_1,Z_2,...,Z_m)$ denote a risk feature distribution. 
We can reasonably expect that LearnRisk is generally effective: if an instance is equivalent (resp. inequivalent), its risk features (excluding its DNN output) are supposed to indicate its equivalence (resp. inequivalence) status. As shown in Eq.3, LearnRisk estimates the equivalence probability expectation of an instance by a weighted linear combination of the expectations of its DNN risk feature and rule risk features. Specifically, given an equivalent pair of $d_i$, ($\mu_i$,$\sigma^2_i$), with $m$ rule risk features, we have
\begin{equation}\label{eq:EffectiveRisk}
\mathbb{E}(\frac{\sum_{j=1}^{m}z_j \cdot w_j \cdot \mu_{f_j}}{\sum_{j=1}^{m}z_j \cdot w_j}) > 0.5,
\end{equation}
in which $f_j$ denotes a rule risk feature, and $\mathbb{E}(*)$ denotes the statistical expectation. Similarly, if $d_i$ is inquivalent, it satisfies
\begin{equation}\label{eq:EffectiveRisk-2}
\mathbb{E}(\frac{\sum_{j=1}^{m}z_j \cdot w_j \cdot \mu_{f_j}}{\sum_{j=1}^{m}z_j \cdot w_j}) < 0.5.
\end{equation}

According to Eq.~\ref{eq:EffectiveRisk} and ~\ref{eq:EffectiveRisk-2}, once a pair is correctly labeled by a classifier, it can be expected that its label would not be flipped by risk-based fine-tuning. Our experiments on real data have confirmed that risk-based fine-tuning rarely flips the labels of true positives and true negatives. Therefore, in the rest of this subsection, we focus on showing that given a mispredicted instance, risk-based fine-tuning can result in a value of $\mu_i$ consistent with its ground-truth label with a fairly good chance.

For theoretical analysis, since both true positives and false negatives (resp. true negatives and false postives) are equivalent (resp. inequivalent) instances, they are assumed to share the same distribution of risk feature activation.  
Formally, we state the assumption on risk feature distribution as follows:

\begin{assumption}
\label{assumption:identical-risk-feature}
{\bf Identicalness of Risk Feature Distributions.}	Given an ER workload $D^t$, the risk feature activation of each equivalent instance $d_i^+$, denoted by $\mathbf{Z}_i^+$, is supposed to follow the same distribution of $\mathbf{Z}^+$; similarly, the risk feature activation of each inequivalent instance $d_i^-$, denoted by $\mathbf{Z}_i^-$, is supposed to follow the same distribution of $\mathbf{Z}^-$. 

\end{assumption}

Based on Assumption~\ref{assumption:identical-risk-feature}, we can establish the lower bound of the estimated equivalence probability expectation of a false negative by the following theorem:
\begin{theorem}
\label{theorem:fn}
	Given a false negative $\tilde{d}_j^{-}$, suppose that there are totally $n$ true positives, denoted by $d_i^+$, ranked after $\tilde{d}_j^{-}$ by LearnRisk such that each true positive, $d_i^+$, satisfies 
\begin{equation}	
		\Delta VaR^{-} - \Delta C^{-} > \epsilon, 
\end{equation} 		
in which 	$\Delta VaR^{-}=VaR^-(\tilde{d}_j^{-}) - VaR^+(d_i^+)$, and $\Delta C^{-} = \mathbb{E}(\mu_{d_i^{+}}-2\sigma_{d_i^{+}})- \mathbb{E}(\mu_{\tilde{d}_j^{-}}-2\sigma_{\tilde{d}_j^{-}})$. Then, for any $\delta \in (0, 1)$, with probability at least $1-\delta$, its expectation of equivalence probability of $\tilde{d}_j^{-}$, $\mu_{\tilde{d}_j^{-}}$, estimated by \emph{LearnRisk}, satisfies 
	\begin{equation}
		\mu_{\tilde{d}_j^{-}} \geq \frac{1}{2} + \frac{\epsilon}{2} - \sqrt{\frac{m + 1}{2}ln[\frac{1}{1-(1-\delta^{\frac{1}{n}})^{\frac{1}{2}}}]},
	\end{equation}
in which $\mu_{*}$ denotes the expectation of equivalence probability and $\sigma_{*}$ denotes its standard deviation.
\end{theorem}

\begin{proof}
	The proofs can be found in Appendix~\ref{appendix:proof-fn}.
\end{proof}

In Theorem~\ref{theorem:fn}, $m$ denotes the number of rule risk features, and the value of $\Delta C^{-}$ corresponds to the difference of risk expectation between false negatives being labeled as \emph{matching} and true positives being labeled as \emph{matching}. Note that the total number of rule risk features ($m$) is usually limited (e.g., dozens or hundreds), while $n$ is usually much larger than $m$. It can be observed that in Theorem~\ref{theorem:fn}, by the exponential effect of $n$, the 3rd term on the right-hand side tends to become zero as
the value of $n$ increases. Therefore, if $\epsilon > 0$ and there are sufficient true positives satisfying the specified condition, risk-based fine-tuning would have a fairly good chance to correctly flip the label of $\tilde{d}_j^{-}$ from \emph{unmatching} to \emph{matching}. To gain deeper insight into Theorem~\ref{theorem:fn}, we analyze the value of $\Delta VaR^{-} - \Delta C^{-}$. Since the optimization objective of LearnRisk is to maximize the risk difference between $VaR^-(\tilde{d}_j^{-})$ and $VaR^+(d_i^{+})$, $\Delta VaR^{-}$ can be expected to large for most true positives. Therefore, we analyze the value of $\Delta C^{-}$. Based on Assumption~\ref{assumption:identical-risk-feature}, we have the following lemma:

\begin{lemma}
\label{lemma:fn-deltac}
	\begin{equation}
		\begin{split}
		\Delta C^{-} \leq &max\big{\{}\mathbb{E}(w_{d_i^{+}}(\hat{\mu}_{d_i^{+}}-2\hat{\sigma}_{d_i^+}))-
		\mathbb{E}(w_{\tilde{d}_j^{-}}(\hat{\mu}_{\tilde{d}_j^{-}}-2\hat{\sigma}_{\tilde{d}_j^{-}})), \\ &\mathbb{E}(w_{d_i^{+}}\hat{\mu}_{d_i^{+}}) -\mathbb{E}(w_{\tilde{d}_j^{-}}\hat{\mu}_{\tilde{d}_j^{-}})\big{\}},
		\end{split}
	\end{equation}
where the $\hat{\mu}_{*}$ and $\hat{\sigma}_{*}$ denote the DNN output probability and its corresponding standard deviation respectively, $w_*$ denotes the learned weight of DNN risk feature.
\end{lemma}

\begin{proof}
	The proofs can be found in Appendix~\ref{appendix:proof-fn-deltac}.
\end{proof}

It is interesting to point out that as shown in Lemma~\ref{lemma:fn-deltac}, the value of $\Delta C^{-}$ only depends on the distributions of DNN outputs and their weights, but independent of the distributions of rule risk features. It has the simple upper bound of 
\begin{equation} \label{eq:DeltaC}
   \Delta C^{-} \leq \mathbb{E}(w_{d_i^{+}}). 
\end{equation}   
Hence, when the learned weight of DNN output becomes smaller, which means that the DNN becomes less accurate, true positives would have a higher chance to satisfy $\Delta VaR^{-} - \Delta C^{-} > 0$. In our experiments, it is observed that the expected weight of classifier output is usually between 0.2 and 0.6, or $0.2\leq \mathbb{E}(w_{d_i^{+}})\leq 0.6$. As a result, Theorem~\ref{theorem:fn} shows that a false negative has a fairly good chance to be flipped from \emph{unmatching} to \emph{matching}.

Based on Assumption~\ref{assumption:identical-risk-feature}, the theoretical chance of a false positive being flipped from \emph{matching} to \emph{unmatching} can be similarly established. The corresponding theorem and lemma are presented as follows.

\begin{theorem}
	\label{theorem:fp}
	Given a false positive $\tilde{d}_j^{+}$, suppose that there are totally $n$ true negatives, denoted by $d_i^-$, ranked after $\tilde{d}_j^{+}$ by LearnRisk such that each true negative, $d_i^-$, satisfies 
	\begin{equation}	
	\Delta VaR^{+} - \Delta C^{+} > \epsilon, 
	\end{equation} 		
	in which 	$\Delta VaR^{+}=VaR^+(\tilde{d}_j^{+}) - VaR^-(d_i^{-})$, and $\Delta C^{+} = \mathbb{E}(\mu_{{\tilde{d}_j^{+}}}+2\sigma_{\tilde{d}_j^{+}})- \mathbb{E}(\mu_{d_i^{-}}+2\sigma_{d_i^{-}})$. Then, for any $\delta \in (0, 1)$, with probability at least $1-\delta$, its expectation of equivalence probability of $\tilde{d}_j^{+}$, $\mu_{\tilde{d}_j^{+}}$ estimated by \emph{LearnRisk}, satisfies 
	\begin{equation*}
	\mu_{\tilde{d}_j^{+}} \leq \frac{1}{2} - \frac{\epsilon}{2} + \sqrt{\frac{m+1}{2}ln[\frac{1}{1-(1-\delta^{\frac{1}{n}})^{\frac{1}{2}}}]},
	\end{equation*}
	in which $\mu_{*}$ denotes the mean of equivalence probability and $\sigma_{*}$ denotes its standard deviation.
\end{theorem}
\begin{proof}
	The proofs can be found in Appendix~\ref{appendix:proof-fp}.
\end{proof}

In Theorem~\ref{theorem:fp}, the value of $\Delta C^{+}$ corresponds to the difference of risk expectation between false positives being labeled as \emph{unmatching} and true negatives being labeled as \emph{unmatching}. Similar to the case of Theorem~\ref{theorem:fn}, it can be observed that in Theorem~\ref{theorem:fp}, by the exponential effect of $n$, the 3rd term on the right-hand side tends to become zero as the value of $n$ increases. Therefore, if $\epsilon > 0$ and there are sufficient true negatives satisfying the specified condition, risk-based fine-tuning would have a fairly good chance to correctly flip the label of $\tilde{d}_j^{+}$ from \emph{matching} to \emph{unmatching}.

To gain a deeper insight into Theorem~\ref{theorem:fp}, we also analyze the value of $\Delta C^{+}$ by the following lemma:

\begin{lemma}
\label{lemma:fp-deltac}
	\begin{equation}
	\begin{split}
	\Delta C^{+} \leq max\big{\{}\mathbb{E}(w_{\tilde{d}_j^{+}}(\hat{\mu}_{\tilde{d}_j^{+}}+2\hat{\sigma}_{\tilde{d}_j^{+}}))-
	\mathbb{E}(w_{d_i^-}(\hat{\mu}_{d_i^-}+2\hat{\sigma}_{d_i^-})), \mathbb{E}(w_{\tilde{d}_j^{+}}\hat{\mu}_{\tilde{d}_j^{+}}) -\mathbb{E}(w_{d_i^-}\hat{\mu}_{d_i^-})\big{\}},
	\end{split}
	\end{equation}
	where the $\hat{\mu}_{*}$ and $\hat{\sigma}_{*}$ denote the DNN output probability and its corresponding standard deviation respectively, $w_*$ denotes the learned weight of DNN risk feature.
\end{lemma}
\begin{proof}
	The proofs can be found in Appendix~\ref{appendix:proof-fp-deltac}.
\end{proof}

Similarly, as shown in Lemma~\ref{lemma:fp-deltac}, the value of $\Delta C^{+}$ has an upper bound constrained by the weights of DNN outputs. The true negatives would tend to satisfy $\Delta VaR^{+} - \Delta C^{+} > 0$ in the case that the trained DNN model becomes less accurate. Hence, Theorem~\ref{theorem:fp} shows that a false positive has a fairly good chance to be flipped form \emph{matching} to \emph{unmatching}.

\vspace{0.2in}
\hspace{-0.15in}{\bf Empirical Validation.} We have illustrated the efficacy of theoretical analysis on the real literature dataset of DBLP-ACM\footnote{\mbox{https://github.com/anhaidgroup/deepmatcher/}\label{dataset}.} The results on the first iteration of risk-based fine-tuning are presented in Table~\ref{tab:da-statistics}, in which the false negatives (resp. false positives) are clustered according to the size of true positives (resp. true negatives) that meet the specified condition. It can be observed: 1) risk-based fine-tuning rarely flips the labels of true positives and true negatives; 2) the majority of false negatives (resp. false positives) have a large number (e.g. $\geq 100$) of corresponding true positives (resp. true negatives), and most of them are correctly flipped.

\begin{table}
	\caption{Empirical Validation of Theoretical Analysis.}
	\label{tab:da-statistics}
	\begin{center}
	
	\subtable[on True Positives and True Negatives]{
		\begin{tabular}{lcccc}
			\toprule
			 & &\# & &\# Flipped \\
			\midrule
			True Positives && 1143 && 12 \\
			True Negatives && 5987 && 27 \\
			\bottomrule
		\end{tabular}
	}
		
	\subtable[on False Negatives and False Positives]{
		\begin{tabular}{lccccc}
			\toprule
			\multirow{1}{*}{\#$(\Delta VaR^{-} > \Delta C^{-}$} & \multicolumn{2}{c}{False Negatives} & & \multicolumn{2}{c}{False Positives}\\
			\cline{2-3}\cline{5-6}
			\multirow{1}{*}{or $\Delta VaR^{+} > \Delta C^{+})$} & \# & \# Flipped & & \# & \# Flipped \\
			\midrule
			\#$< 100$ & 1 & 0 & & 0 & 0\\
			\#$\ge 100$ & 188 & 179 & & 100 & 91\\
			Total & 189 & 179 & & 100 & 91\\
			\bottomrule
		\end{tabular}
	}
\end{center}
\end{table}

	\section{Experiments}
\label{sec:experiments}

In this section, we empirically evaluate the proposed approach on real benchmark datasets by a comparative study. We first describe the experimental setting. Then, we present the comparative evaluation results. Finally, we evaluate robustness of the proposed approach w.r.t the size of validation data.

\subsection{Experimental Setup} \label{sec:exp-setup}
 
We have used six real datasets from three domains in our empirical study: 

\begin{itemize}
\item \textbf{Publication}. The datasets in this domain contain bibliographic data from different sources, i.e. DBLP, Google Scholar and ACM. As in~\cite{mudgal2018deep}, we used DBLP-Scholar\footnote{\mbox{https://github.com/anhaidgroup/deepmatcher/}\label{dataset}} (denoted by DS) and DBLP-ACM~\textsuperscript{\ref{dataset}} (denoted by DA). Additionally, we used the Cora dataset\footnote{\mbox{http://www.cs.utexas.edu/users/ml/riddle/data/cora.tar.gz}}, which contains the citation data obtained from the Cora search engine; 
\item \textbf{Music}. In this domain, we used the Itunes-Amazon dataset (denoted by IA) provided by~\cite{mudgal2018deep}. The size of IA is relatively small, containing only 539 pairs. Additionally, we used the Songs dataset\footnote{\mbox{http://pages.cs.wisc.edu/\~anhai/data/falcon\_data/songs/}} (denoted by SG), which contains song records; the experiments match the entries within the same table;
\item \textbf{Product}. In this domain, we used a dataset containing the electronics product pairs extracted from the Abt.com and Buy.com~\textsuperscript{\ref{dataset}}. We denote this dataset by AB.
\end{itemize}

\begin{table}[ht]
	\caption{The statistics of datasets.}
	\label{tab:datasets}
	\begin{center}
		\begin{small}
			\begin{sc}
				\begin{tabular}{lccc}
					\toprule
					Dataset & Size & \# Matches & \# Attributes \\
					\midrule
					DS & 28,707 & 5,347 & 4 \\
					DA & 12,363 & 2,220 & 4 \\
					Cora & 12,674 & 3,268 & 12 \\
					AB & 9,575 & 1,028 & 3 \\
					IA & 539 & 132 & 8 \\
					SG & 19,633 & 6,108 & 7 \\
					\bottomrule
				\end{tabular}
			\end{sc}
		\end{small}
	\end{center}
\end{table}

   As usual, on all the datasets, we used the blocking technique to filter the pairs deemed unlikely to match. The datasets of DS, DA, AB and IA have been made online available at~\textsuperscript{\ref{dataset}}. On both Cora and SG, we first filtered the pairs and then randomly selected a proportion of the resulting candidates to generate the workloads. The statistics of the test datasets are given in Table~ \ref{tab:datasets}.

We primarily used DeepMatcher~\citep{mudgal2018deep}, the classical deep learning solution for ER, as the classifier. We evaluate the
proposed approach both scenarios where training and test data come from the same source and they come from different sources, thus resulting in more distribution misalignment. In the scenario where training and test data come from the same source, we randomly split each dataset into three parts by a ratio (e.g. 2:2:6 in our experiments) as in~\citet{mudgal2018deep}, which specifies the proportions of training, validation and test set respectively. We evaluate the performance of the proposed approach w.r.t different sufficiency levels of training data. Since DeepMatcher performs very well on DA and SG with only 20\% of their data as training data, we randomly select 10\%, 30\%, 50\%, 70\% and 100\% of the split set of training data to simulate different sufficiency levels. On AB, we fix the proportion of validation data at 20\% and vary the proportions of training and test data from (60\%,20\%) to (20\%,60\%), resulting in totally 5 sufficiency levels. In this scenario, since training and target data are randomly selected from the same source, we compare the \emph{Risk} approach with the original DeepMatcher, which is denoted by \emph{Tradition}.

In the scenario of distribution misalignment, we used the three datasets in the domain of \emph{publication} (i.e., DS, DA and Cora) to generate six pairwise workloads. For instance, DA2DS denotes the workload where training data come from DA while validation data and test data come from DS. On all the workloads, validation and test data are randomly selected from the original target dataset with both percentages set at 20\%. 
In this scenario, besides \emph{Tradition}, we also compare \emph{Risk} with the state-of-the-art technique of transfer learning for ER proposed in ~\citet{DBLP:conf/acl/KasaiQGLP19}. We denote this approach by \emph{Transfer}. It inserted a dataset classifier into the DeepMatcher structure, which can force a deep model to focus on the parameters that are shared by both training and test data.

We have implemented the proposed solution on \emph{DeepMatcher} by replacing the original loss function with the risk-based loss function. To overcome the randomness caused by model initialization and training data shuffling, on each experiment, we perform 5 training sessions and report their mean F1-score on test data. In \emph{Tradition} and \emph{Transfer}, each training session consists of 20 iterations; in \emph{Risk}, the traditional training phase consists of 20 iterations and the risk-based training phase consists of 10 iterations. Our experiments showed that further increasing the number of iterations in each session had very marginal impact on performance.

Additionally, we have also implemented and evaluated the proposed solution on \emph{Ditto}~\citep{li2020deep}, which is the state-of-the-art DNN for ER based on pre-trained Transformer-based language models.  Note that compared with \emph{DeepMatcher}, \emph{Ditto} generally performs better and can perform well with less training data. Therefore, on AB and IA, our experiments begins with the ratio of training data at 10\%.
Similarly, we perform 5 training sessions and report their mean F1-score on test data to overcome the randomness. We use the default parameters of \emph{Ditto} for both traditional training and adaptive training, except that the learning rate of risk-based training is set to be $3*10^{-6}$, instead of the default $3*10^{-5}$. Our implementations on both \emph{DeepMatcher} and \emph{Ditto} and the test data have been made open-source at our website~\footnote{\mbox{https://chenbenben.org/adaptive-training.html}}.

\subsection{Comparative Evaluation on DeepMatcher} \label{sec:exp-comparison}

 \subsubsection{Same-source Scenario}

\begin{figure*}[ht]
	\begin{center}
		\subfigure[DS.]
		{\includegraphics[width=0.31\linewidth]{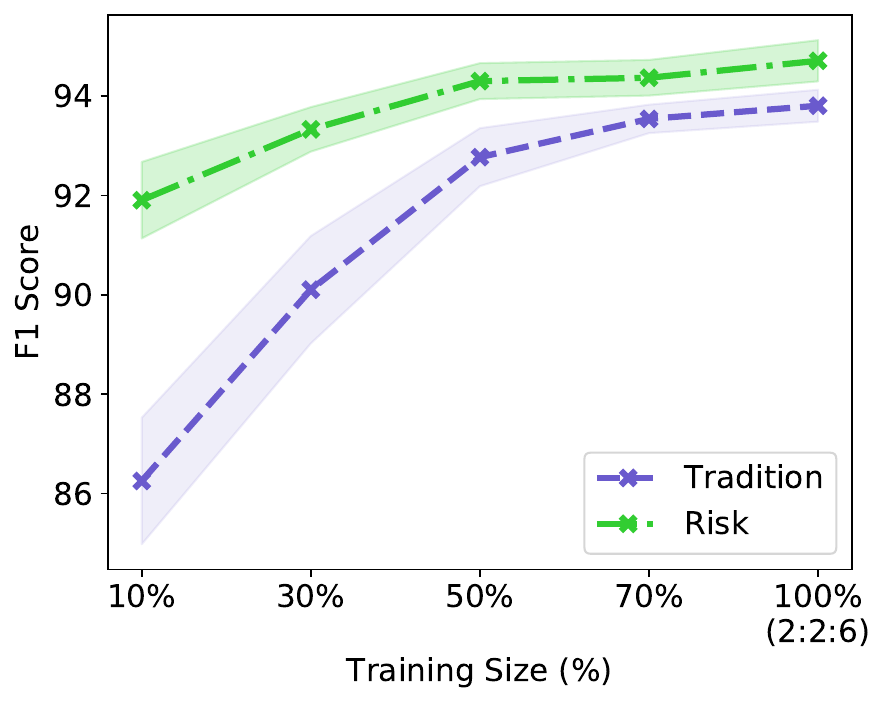}}
		\subfigure[DA.]
		{\includegraphics[width=0.31\linewidth]{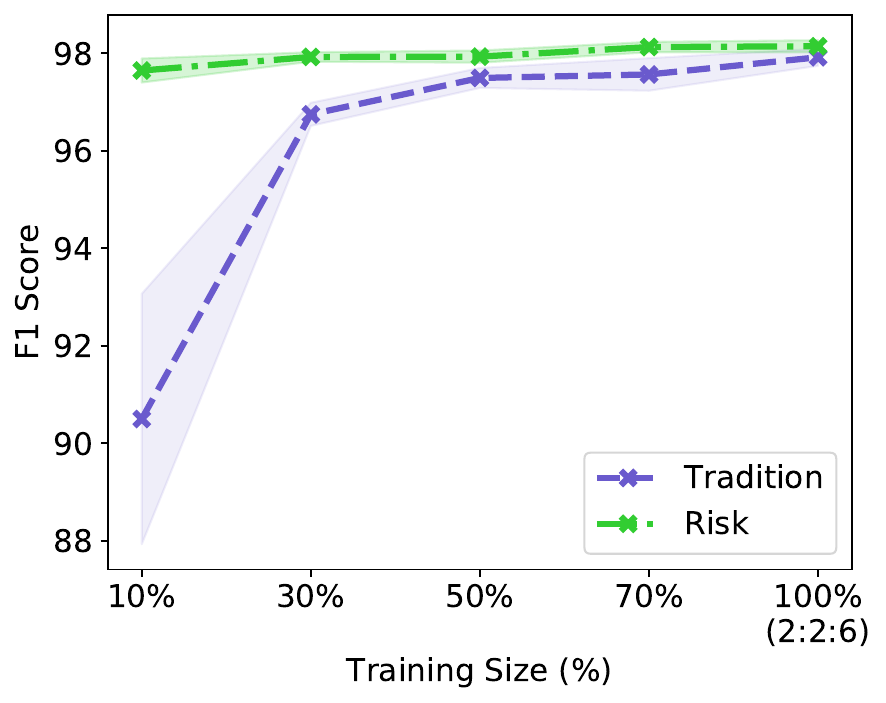}}
		\subfigure[Cora.]
		{\includegraphics[width=0.31\linewidth]{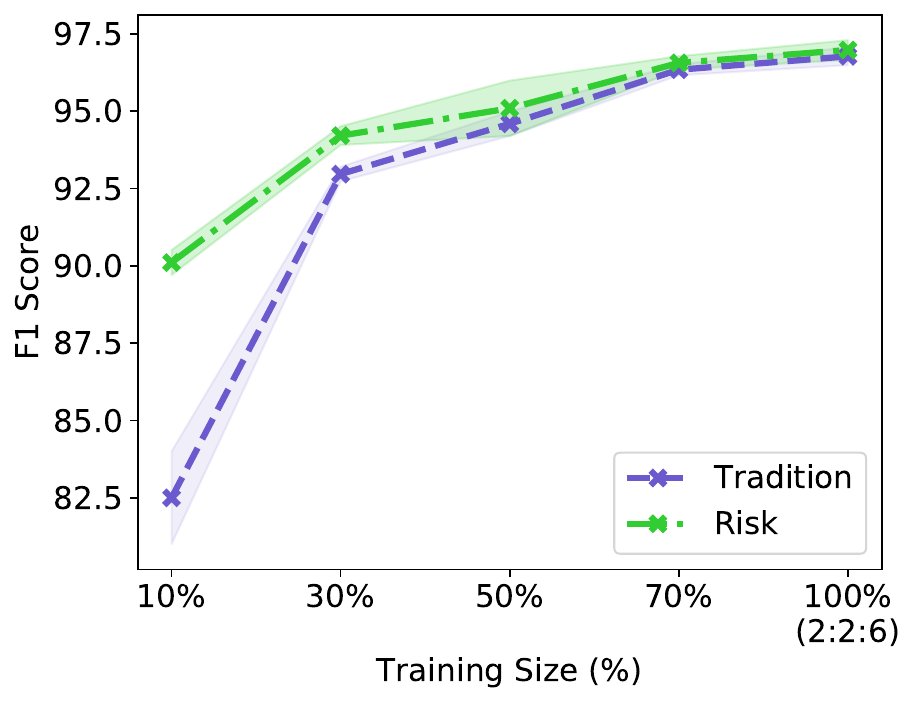}}
		\subfigure[AB.]
		{\includegraphics[width=0.31\linewidth]{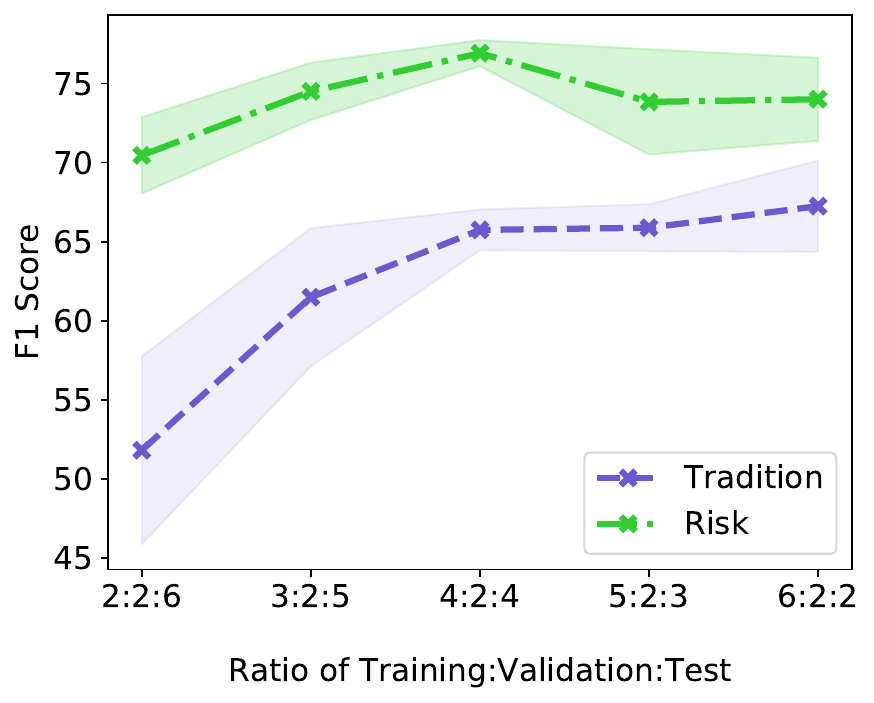}}
		\subfigure[IA.]
		{\includegraphics[width=0.31\linewidth]{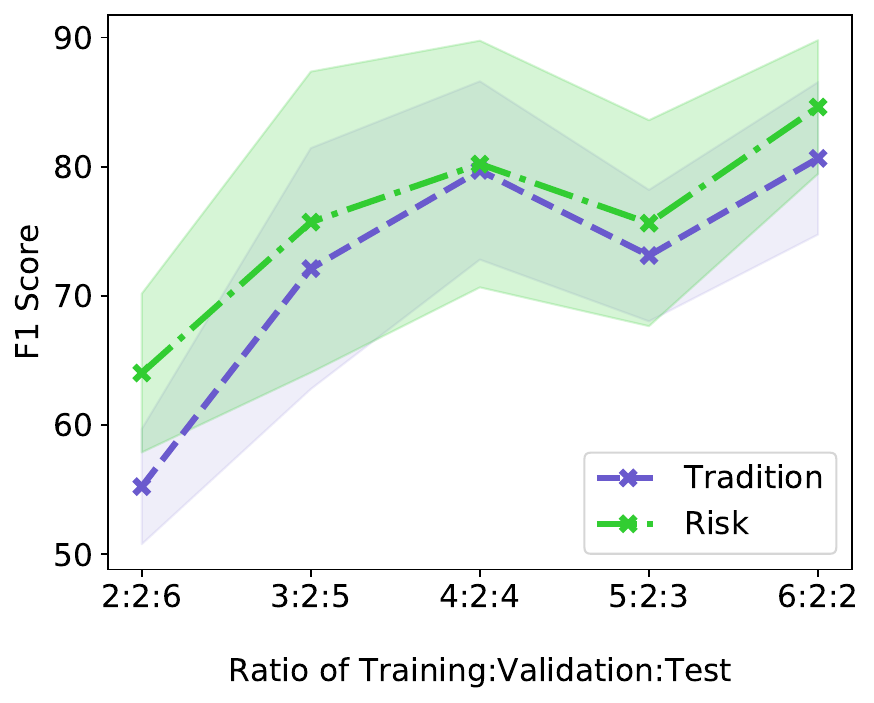}}
		\subfigure[SG.]
		{\includegraphics[width=0.31\linewidth]{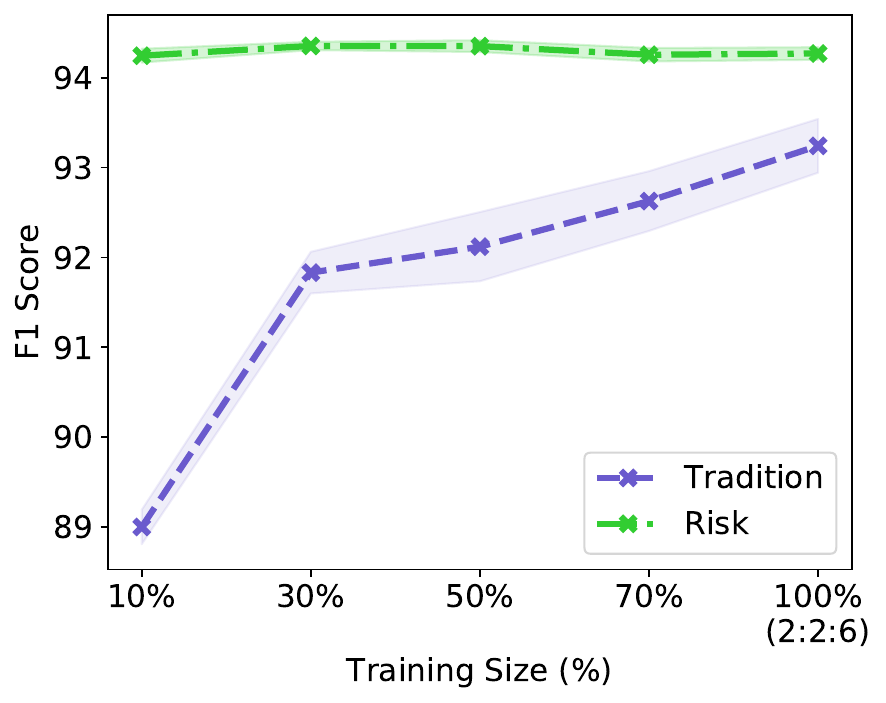}}
		\vskip -0.1in
		\caption{Comparative Evaluation with Deepmatcher: the Same-Source Scenario.}
		\label{fig:exp-comparison}
	\end{center}
	\vskip -0.15in
\end{figure*}

The comparative results are presented in Figure~\ref{fig:exp-comparison}, in which we report both the mean of F1 score and its standard deviation (represented by the shadow in the figure). It can be observed that \emph{Risk} achieves consistently better performance than \emph{Tradition}. In the circumstances where \emph{Tradition} performs unsatisfactorily (e.g. AB and IA), the performance margins between \emph{Risk} and \emph{Tradition} are very considerable. For instance, on AB, with the ratio of (2, 2, 6), \emph{Risk} achieves a performance improvement close to 20\% over \emph{Tradition} (70\% vs 51\%).

In the circumstances where \emph{Tradition} can perform well (e.g. DS, DA, Cora and SG) with the ratio of (2,2,6), it can be observed that the performance margins between \emph{Risk} and \emph{Tradition} are similarly considerable when training data are insufficient. For instance, on DS, with 10 percent of the training data, \emph{Risk} outperforms \emph{Tradition} by more than 6\% and achieves the F1 score of more than 92\%. In particular, on DA and SG, with only 10\% of the training data, \emph{Risk} achieves the performance very similar to what is achieved by using 100\% of the training data. As the size of training data increases, the margins between \emph{Risk} and \emph{Tradition} tend to decrease. This trend can be expected, because when training and test data are randomly selected from the same source, more training data mean less improvement potential for risk-based fine-tuning. Due to the small size of IA, the comparative results on IA have higher randomness compared with other datasets.

\subsubsection{Scenario of Distribution Misalignment} \label{sec:exp-ood}

The comparative results are presented in Table~\ref{tab:ood}, in which the best results have been highlighted.  It can be observed that: 1) The performance of \emph{Tradition} deteriorates significantly in most testbeds; 2) the performance of \emph{Transfer} also fluctuates wildly across the test workloads. As shown on DA2CORA and CORA2DA, its impact becomes very marginal or even negative when \emph{Tradition} performs decently; 3) \emph{Risk} consistently outperforms both \emph{Tradition} and \emph{Transfer}, and the margins are very considerable in most cases. 

 We further explain the efficacy of the risk-based approach by illustrative examples. On DA2DS, we observed that the model trained on DA performs very poorly (only around 20\%) on the target workload of DS. This is mainly due to the fact that DS is more challenging than DA, and the data distribution of DA to a large extent fails to reflect the more complicated distribution of DS. In contrast, LearnRisk can reliably identify the mispredictions of the pre-trained model on DS. We observed that in the first iteration of risk-based fine-tuning, it correctly identifies totally 877 mispredictions among the top 1000 risky pairs, most of which are later correctly flipped. However, on the workloads (e.g. DS2DA and DS2CORA) where \emph{Tradition} performs well, the advantage of \emph{Risk} over \emph{Tradition} becomes less significant. This result should be no surprise because in this circumstance, risk analysis becomes more challenging. 
 
 \begin{table}[t]
 	\caption{Comparative Evaluation: Distribution Misalignment.}
 	\label{tab:ood}
 	\begin{center}
 		\begin{small}
 			\begin{sc}
 				\begin{tabular}{lccc}
 					\toprule
 					\multirow{2}{*}{Dataset} & \multicolumn{3}{c}{F1 Score (Mean $\pm$ Standard deviation)} \\
 					\cline{2-4}
 					& Baseline & Transfer & Risk\\
 					\midrule
 					DA2DS & $19.86\pm5.12$ & $43.81\pm11.88$  & {\bf 91.67}$\pm${\bf 0.56} \\
 					DA2Cora &  $76.47\pm4.59$ & $74.86\pm3.70$  & {\bf 89.08}$\pm${\bf 0.81} \\
 					Cora2DS &  $55.81\pm5.90$ & $62.08\pm6.65$  & {\bf 86.55}$\pm${\bf 1.34} \\
 					Cora2DA &  $71.28\pm5.23$ & $72.92\pm7.57$  & {\bf 96.99}$\pm${\bf 0.34} \\
 					DS2DA & $93.08\pm1.71$  & $93.50\pm1.49$  &  {\bf 94.18}$\pm${\bf 0.97} \\
 					DS2Cora & $83.11\pm1.81$  & $82.48\pm3.52$  & {\bf 84.88}$\pm${\bf 0.29} \\
 					\bottomrule
 				\end{tabular}
 			\end{sc}
 		\end{small}
 	\end{center}
 \end{table}

\subsection{Comparative Evaluation on Ditto} \label{sec:exp-ditto}

\begin{figure*}[ht]
	\begin{center}
		\subfigure[DS.]
		{\includegraphics[width=0.31\linewidth]{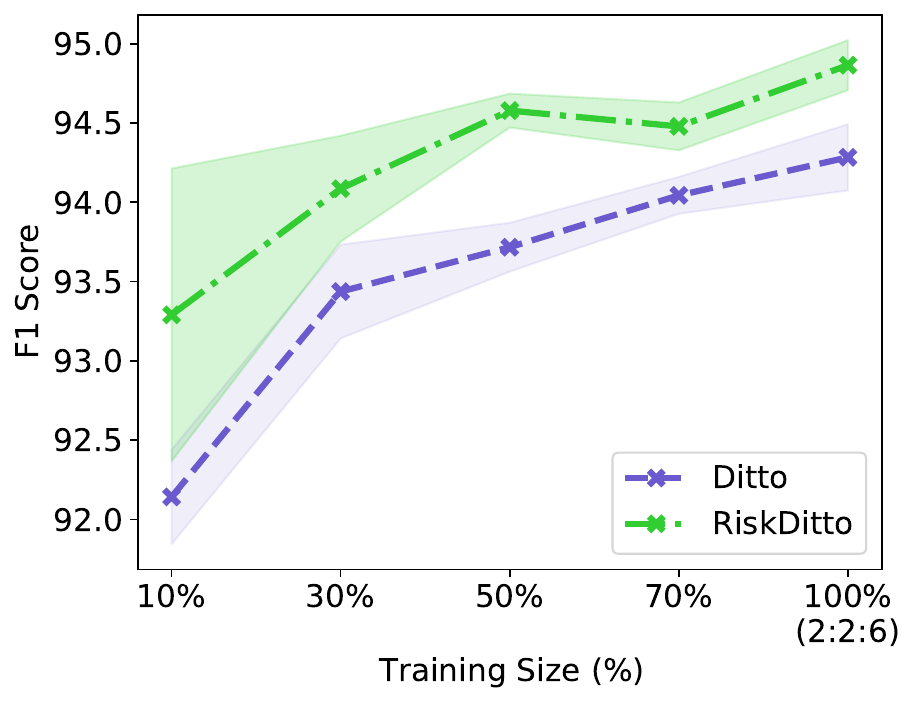}}
		\subfigure[DA.]
		{\includegraphics[width=0.31\linewidth]{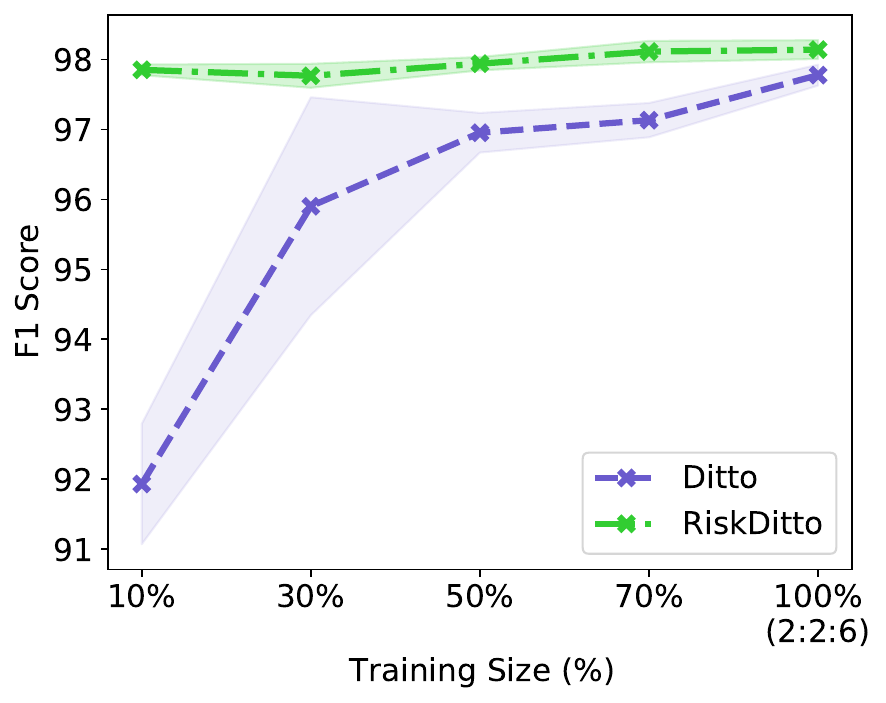}}
		\subfigure[Cora.]
		{\includegraphics[width=0.31\linewidth]{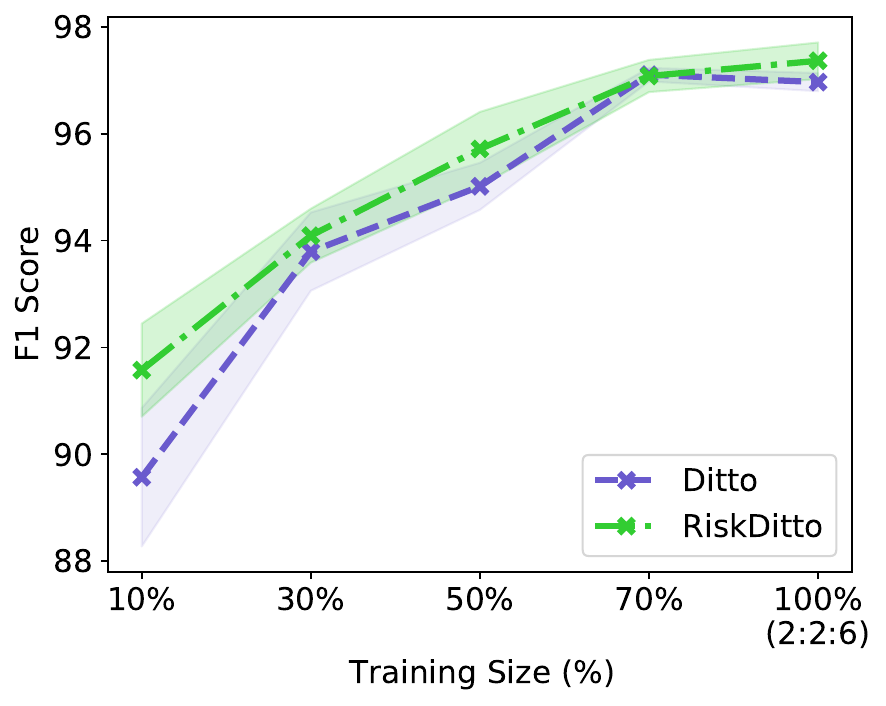}}
		\subfigure[AB.]
		{\includegraphics[width=0.31\linewidth]{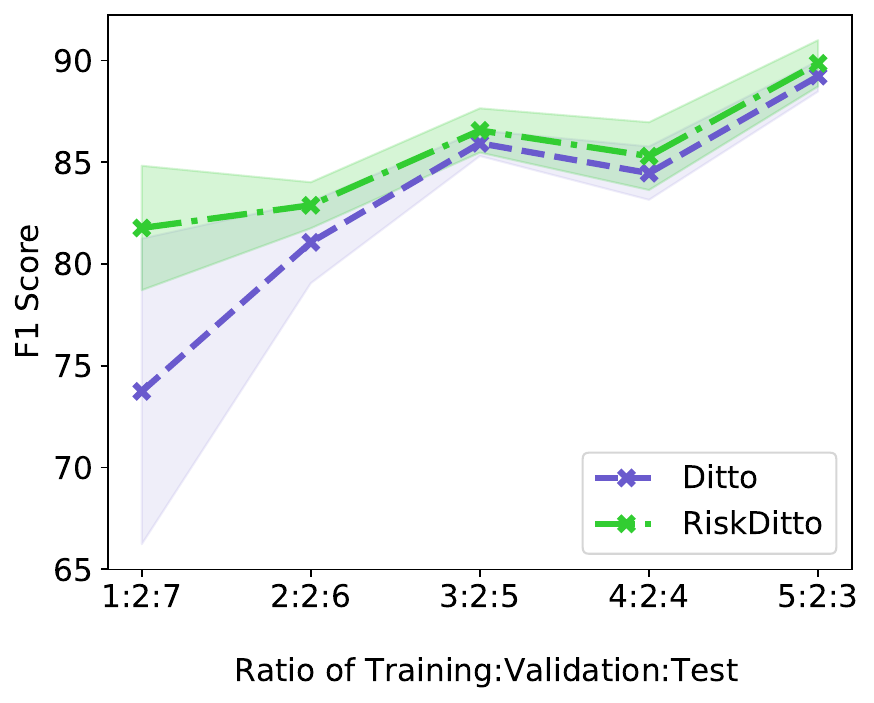}}
		\subfigure[IA.]
		{\includegraphics[width=0.31\linewidth]{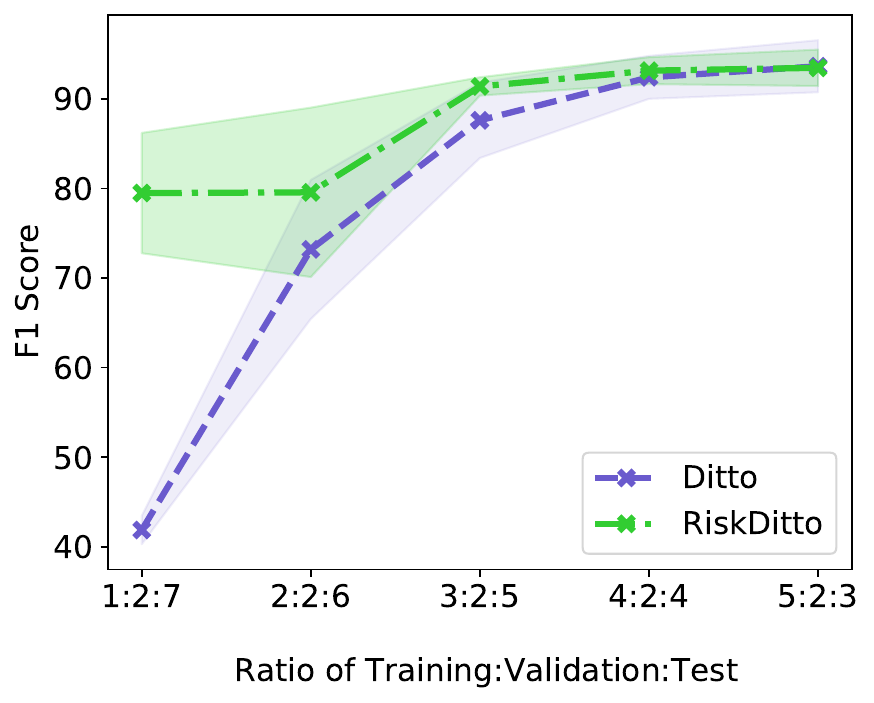}}
		\subfigure[SG.]
		{\includegraphics[width=0.31\linewidth]{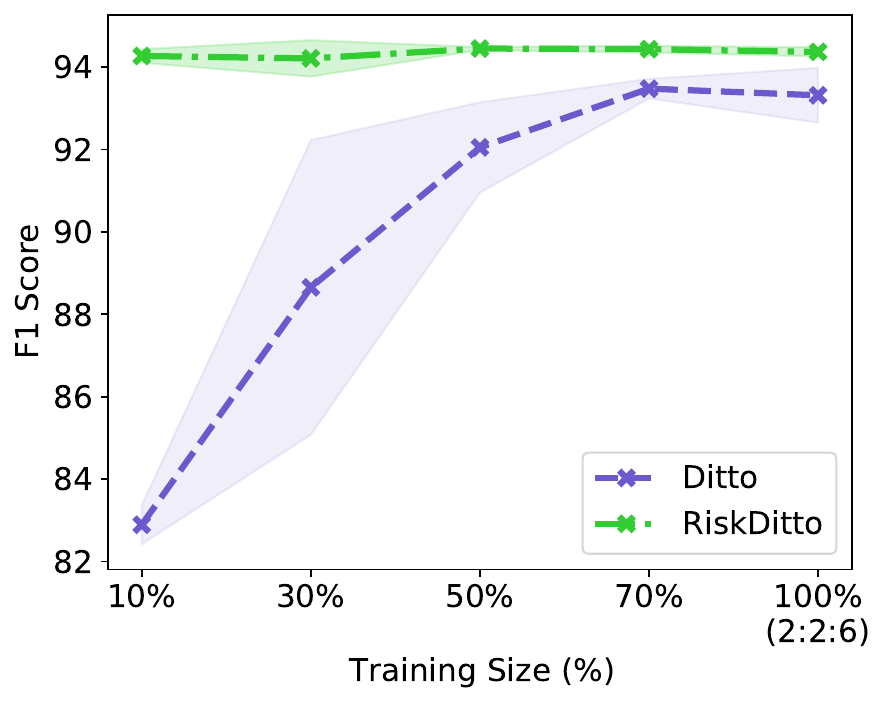}}
		\vskip -0.1in
		\caption{Comparative Evaluation with Ditto: the Same-Source Scenario.}
		\label{fig:exp-ditto}
	\end{center}
	\vskip -0.15in
\end{figure*}

 The comparative results on \emph{Ditto} are presented in Figure~\ref{fig:exp-ditto}.  It can be observed that \emph{Ditto} generally performs better than \emph{Deepmatcher} with the most improvements on AB. On AB, with the ratio setting of (2, 2, 6), the F1 score of \emph{Ditto} is 81\% while \emph{Deepmatcher} can only achieve 51\%. Similar to what have been observed on \emph{DeepMatcher}, risk-based fine-tuning effectively improves the performance of \emph{Ditto} even though it is a better baseline. For instance, on SG, with the percentage of training data at 10\% and 30\%, the performance improvements measured by F1 are 11\% and 6\% respectively. The evaluation results on \emph{Ditto} clearly demonstrate that the proposed approach of risk-based adaptive training is generally applicable to various DNN models.

\subsection{Robustness w.r.t Size of Validation Data} \label{sec:exp-sensitivity}
 
 \begin{figure}
 	\begin{center}
 		\subfigure[DA.]
 		{\includegraphics[width=0.31\columnwidth]{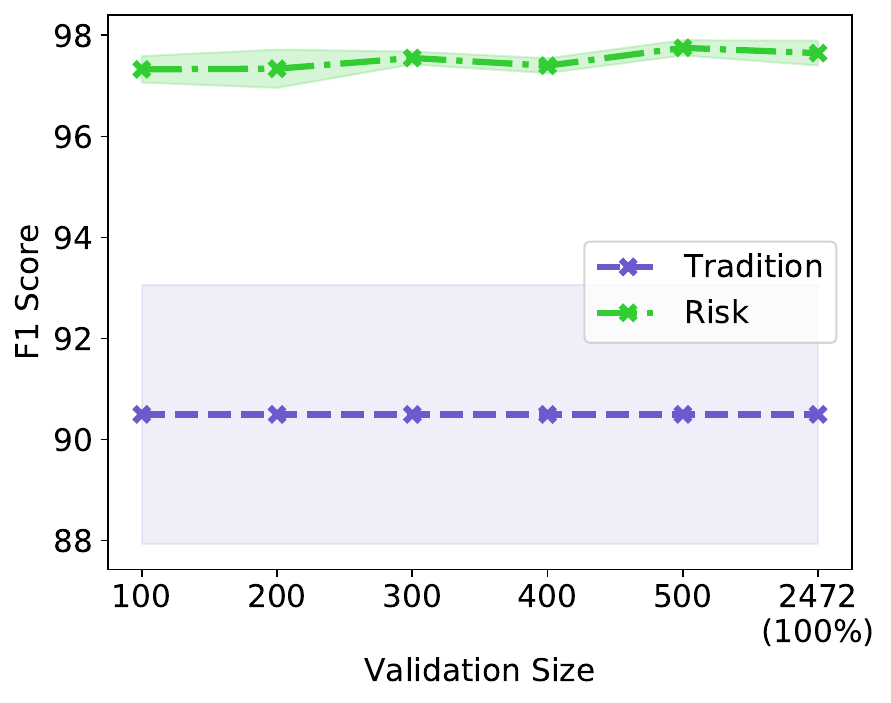}}
 		\subfigure[AB.]
 		{\includegraphics[width=0.31\columnwidth]{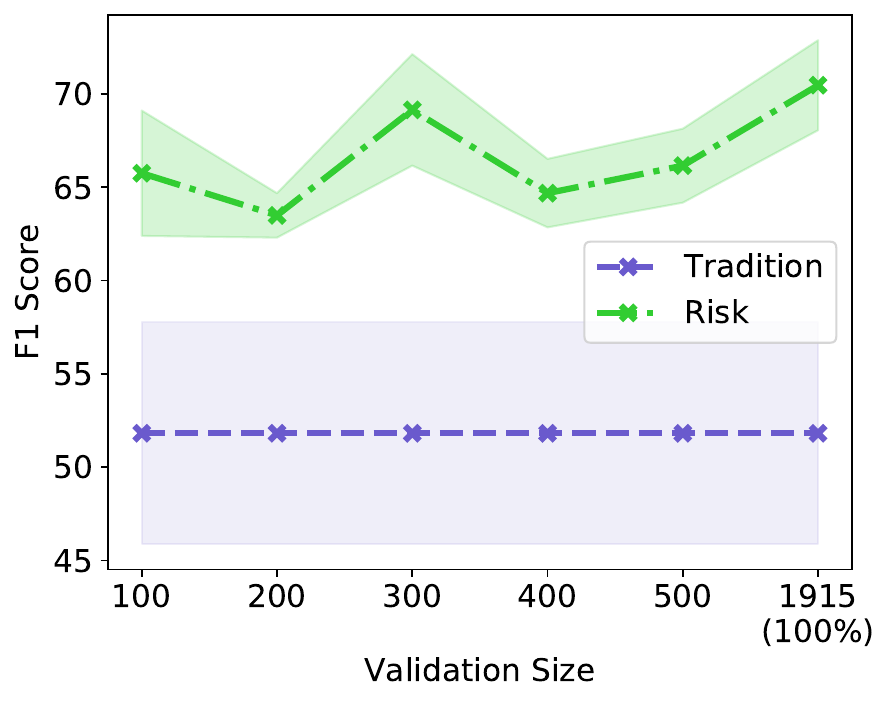}}
 		\subfigure[SG.]
 		{\includegraphics[width=0.31\columnwidth]{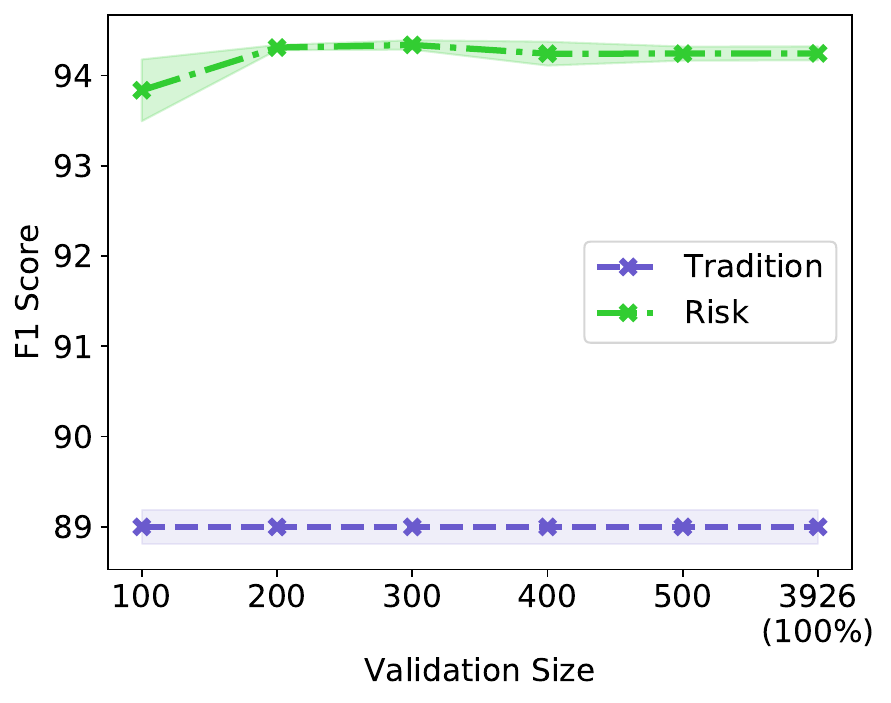}}
 		\caption{Robustness Evaluation.}
 		\label{fig:sensitivity}
 	\end{center}
 \end{figure}

In real scenarios, validation data are necessary for hyperparameter tuning and model selection to ensure that a trained model can generalize well. However, due to labeling cost, it is usually desirable to reduce the size of validation data. Since risk analysis leverages validation data for risk model learning, we evaluate the performance robustness of the proposed approach w.r.t the size of validation data.

To this end, we fixed the sets of training and test data at 20\% and 60\% respectively, and varied the size of validation data by randomly selecting part of instances from the split set of validation data. The results on the DA, AB and SG workloads are presented in Figure~\ref{fig:sensitivity}, in which the performance of \emph{Tradition} and \emph{Risk} with the whole validation data are also included for reference. The evaluation for DA and SG is based on the setting that 10\% of the split set of training data is used. It can be observed that with as few as 100 validation instances, \emph{Risk} is able to improve classifier performance by considerable margins. Our evaluation results are consistent with those reported in ~\citet{chen2019towards}, which showed that the performance of LearnRisk is very robust w.r.t the size of validation data. These experimental results bode well for the application of the proposed approach in real scenarios.

	\section{Conclusion}
\label{sec:conclusion}

In this paper, we have proposed a risk-based approach to enable adaptive deep learning for ER. It can effectively tune deep models towards a target workload by its particular characteristics. Both theoretical analysis and empirical study have validated its efficacy. For future work, it is worthy to point out that the proposed approach is generally applicable to other classification tasks; their technical solutions however need further investigations.

	\newpage
	
	\appendix

\section{Theoretical Analysis}

In theoretical analysis, for simplicity of presentation, without loss of generality, we suppose that \emph{LearnRisk} sets the confidence value at $\theta=0.975$. Hence, given a pair $d_i$ with the equivalence probability distribution of $\mathcal{N}(\mu_i, \sigma_i^2)$, its VaR risk is equal to $1-(\mu_i-2\sigma_i)$ if it is labeled as \emph{matching} by a classifier, and its VaR risk is equal to $\mu_i+2\sigma_i$ if it is labeled as \emph{unmatching}.

%
%
\subsection{Proof of Theorem~\ref{theorem:fn}}
\label{appendix:proof-fn}
	\textbf{Theorem~\ref{theorem:fn}} \emph{Given a false negative $\tilde{d}_j^{-}$, suppose that there are totally $n$ true positives, denoted by $d_i^+$, ranked after $\tilde{d}_j^{-}$ by LearnRisk such that each true positive, $d_i^+$, satisfies 
	\begin{equation}	
	\Delta VaR^{-} - \Delta C^{-} > \epsilon, 
	\end{equation} 		
	in which $\Delta VaR^{-}=VaR^-(\tilde{d}_j^{-}) - VaR^+(d_i^+)$, and $\Delta C^{-} = \mathbb{E}(\mu_{d_i^{+}}-2\sigma_{d_i^{+}})- \mathbb{E}(\mu_{\tilde{d}_j^{-}}-2\sigma_{\tilde{d}_j^{-}})$. Then, for any $\delta \in (0, 1)$, with probability at least $1-\delta$, its expectation of equivalence probability of $\tilde{d}_j^{-}$, $\mu_{\tilde{d}_j^{-}}$, estimated by \emph{LearnRisk}, satisfies 
	\begin{equation*}
	\mu_{\tilde{d}_j^{-}} \geq \frac{1}{2} + \frac{\epsilon}{2} - \sqrt{\frac{m+1}{2}ln[\frac{1}{1-(1-\delta^{\frac{1}{n}})^{\frac{1}{2}}}]},
	\end{equation*}
	in which $\mu_{*}$ denotes the mean of equivalence probability and $\sigma_{*}$ denotes its standard deviation.}

Note that in Theorem~\ref{theorem:fn}, $m$ denotes the number of rule risk features and $\Delta C^{-}$ denotes the difference of risk expectation between false negatives being labeled as \emph{matching} and true positives being labeled as \emph{matching}. In the rest of this section, we first prove a lemma that states the concentration inequalities of VaR risk functions, and then prove Theorem~\ref{theorem:fn} based on the lemma.

\begin{lemma}
\label{lemma:fn}
	Given a randomly selected pair $d_i^+$ from true positives, whose equivalence probability distribution estimated by \emph{LearnRisk} is denoted by $\mathcal{N}(\mu_{d_i^+},\sigma_{d_i^+})$, for any $\delta \in (0, 1)$, with probability at least $(1-\delta)$, the following inequality holds
	\begin{equation*}
	(\mu_{d_i^+}-2\sigma_{d_i^+})-\mathbb{E}(\mu_{d_i^+}-2\sigma_{d_i^+}) \leq \varepsilon,
	\end{equation*}
	where $\varepsilon=\sqrt{\frac{m+1}{2}ln(\frac{1}{\delta})}$. Similarly, for a randomly selected false negative $\tilde{d}_j^{-}$ with equivalence probability distribution of $\mathcal{N}(\mu_{\tilde{d}_j^{-}},\sigma_{\tilde{d}_j^{-}})$, with probability at least $(1-\delta)$, the following inequality holds
	\begin{equation*}
	\mathbb{E}(\mu_{\tilde{d}_j^{-}}-2\sigma_{\tilde{d}_j^{-}}) - (\mu_{\tilde{d}_j^{-}}-2\sigma_{\tilde{d}_j^{-}}) \leq \varepsilon.
	\end{equation*}
\end{lemma}

\begin{proof}
	Consider the randomly selected pair $d_i^+$ from true positives. 
	The mean of its equivalence probability can be represented by
	\begin{equation*}
		\mu_{d_i^+}=\frac{\displaystyle\sum_{k=1}^{m}w_k\mu_{f_k}Z_k + w_{d_i^+}\hat{\mu}_{d_i^+}}{\displaystyle\sum_{k=1}^{m}w_k Z_k + w_{d_i^+}},
	\end{equation*}
	where $m$ is the number of rule risk features, $w_k$ is the learned weight of a risk feature $f_k$, $\mu_{f_k}$ is the probability mean of the feature, $Z_k$ is a random variable indicates if a selected pair has this feature, and $\hat{\mu}_{d_i^+}$ is the output probability by a classifier with its weight $w_{d_i^+}$. Note that for a randomly selected true positive, the values of $w_k$ and $\mu_{f_k}$ for each rule risk feature are fixed, while $Z_k$, $\hat{\mu}_{d_i^+}$ and $w_{d_i^+}$ are random variables. Note that according to LearnRisk, the value of $w_{d_i^+}$ totally depends on $\hat{\mu}_{d_i^+}$. 
	
	The standard deviation of its equivalence probability can also be represented by
	\begin{equation*}
		\sigma_{d_i^+}=\frac{1}{\displaystyle\sum_{k=1}^{m}w_k Z_k + w_{d_i^+}}\sqrt{\displaystyle\sum_{k=1}^{m}w_k^{2}\sigma_{f_k}^{2}Z_k+w^2_{d_i^+}\hat{\sigma}^2_{d_i^+}},
	\end{equation*}
	where $\hat{\sigma}^2_{d_i^+}$ denotes the corresponding variance of a classifier's output $\hat{\mu}_{d_i^+}$.
	
	Recall that a function $f: X^n \to \mathbb{R}$ has the \emph{bounded differences property} if for some non-negative constants $c_1, c_2, ..., c_n$,
	\begin{equation*}
	\mathop{sup}\limits_{x_1,...,x_n,x'_k \in X}{|f(x_1,...,x_{k-1}, x_k, x_{k+1},...,x_n)-f(x_1,...,x_{k-1}, x'_k, x_{k+1},...,x_n)|} \leq c_k, 1 \leq k \leq n
	\end{equation*}
	The bounded differences property shows that if the $i$th variable is changed while all the others being fixed, the value of $f$ will not change by more than $c_k$.
		
	Let $f(Z_1,...,Z_m, \hat{\mu}_{d_i^+})=\mu_{d_i^+}-2\sigma_{d_i^+}$. Now we proceed to consider the bounded differences property of $f$. Note that a valid equivalence probability should be between 0 and 1. Hence, for all $\mu\geq0, \sigma\geq0$, we have $0\leq\mu\pm 2\sigma\leq1$. As a result, by changing the value of $Z_k$, we have
	\begin{equation*}
	sup|f(z_1,...,z_k,...,z_m, \hat{\mu}_{d_i^+})-f(z_1,...,z'_k,...,z_m, \hat{\mu}_{d_i^+})| \leq 1
	\end{equation*}
	
	Similarly, the upper bound of $f$ by changing the value of $\hat{\mu}_{d_i^+}$ is
	\begin{equation*}
	sup|f(z_1,...,z_m, \hat{\mu}_{d_i^+})-f(z_1,...,z_m, \hat{\mu}'_{d_i^+})| \leq 1
	\end{equation*}
	
At this point, we have obtained the upper bounds of the function $f(Z_1,...,Z_m, \hat{\mu}_{d_i^+})$ by changing any one of the variables. Denoting these bounds as $c_1, ..., c_m, c_{m+1}$, where $c_k=1, 1\leq k \leq m+1$, we have
\begin{equation*}
	\displaystyle\sum_{k=1}^{m+1}c_k^2 = m+1
\end{equation*}

Recall that the McDiarmid's inequality \citep{mcdiarmid_1989} states that if a function $f$ satisfies the bounded differences property with constants $c_1,...,c_n$. Let $Y=f(X_1, ..., X_n)$, where the $X_k$s are independent random variables. Then, for all $\varepsilon > 0$,
	\begin{equation*}
	\mathbb{P}(Y-\mathbb{E}Y \geq \varepsilon) \leq exp(-\frac{2\varepsilon^2}{\sum_{k=1}^{n}c_k^{2}});
	\end{equation*}
	\begin{equation*}
	\mathbb{P}(\mathbb{E}Y - Y \geq \varepsilon) \leq exp(-\frac{2\varepsilon^2}{\sum_{k=1}^{n}c_k^{2}}).
	\end{equation*}

Based on the McDiarmid's inequality, for all $\varepsilon > 0$, we have
\begin{equation*}
	\mathbb{P}(\mu_{d_i^+}-2\sigma_{d_i^+}-\mathbb{E}(\mu_{d_i^+}-2\sigma_{d_i^+}) \geq \varepsilon)\leq exp(-\frac{2\varepsilon^2}{\sum_{k=1}^{m+1}c_k^{2}}) = exp(-\frac{2\varepsilon^2}{m+1}).
\end{equation*}

Let $\delta=exp(-\frac{2\varepsilon^2}{m+1})$, we can get that $\varepsilon=\sqrt{\frac{m+1}{2}ln(\frac{1}{\delta})}$. Hence, with the probability at least $1-\delta$, the inequality $\mu_{d_i^+}-2\sigma_{d_i^+}-\mathbb{E}(\mu_{d_i^+}-2\sigma_{d_i^+}) \leq \varepsilon$ holds.

 Similarly,  with the probability at least $1-\delta$, we have $\mathbb{E}(\mu_{\tilde{d}_j^{-}}-2\sigma_{\tilde{d}_j^{-}}) - (\mu_{\tilde{d}_j^{-}}-2\sigma_{\tilde{d}_j^{-}}) \leq \varepsilon$.
\end{proof}

%
%
In the following, we present the proof of Theorem~\ref{theorem:fn}.

\begin{proof}
\textbf{[Theorem~\ref{theorem:fn}]}	According to Lemma~\ref{lemma:fn}, with the probability at least $(1-\delta)^2$, the following inequalities hold
	\begin{equation*}
	\begin{split}
		&(\mu_{d_i^+}-2\sigma_{d_i^+})-\mathbb{E}(\mu_{d_i^+}-2\sigma_{d_i^+})+\mathbb{E}(\mu_{\tilde{d}_j^{-}}-2\sigma_{\tilde{d}_j^{-}}) - (\mu_{\tilde{d}_j^{-}}-2\sigma_{\tilde{d}_j^{-}}) \leq 2\varepsilon; \\
		&(\mu_{d_i^+}-2\sigma_{d_i^+}-\mu_{\tilde{d}_j^{-}}+2\sigma_{\tilde{d}_j^{-}})+\mathbb{E}(\mu_{\tilde{d}_j^{-}}-2\sigma_{\tilde{d}_j^{-}})-\mathbb{E}(\mu_{d_i^+}-2\sigma_{d_i^+}) \leq 2\varepsilon;
	\end{split}
	\end{equation*}
	
Hence, we have	
	\begin{equation}
		\mu_{d_i^+}-2\sigma_{d_i^+}-\mu_{\tilde{d}_j^{-}}+2\sigma_{\tilde{d}_j^{-}} \leq 2\varepsilon+\big{[}\mathbb{E}(\mu_{d_i^+}-2\sigma_{d_i^+})- \mathbb{E}(\mu_{\tilde{d}_j^{-}}-2\sigma_{\tilde{d}_j^{-}})\big{]},
	\end{equation}
	where $\varepsilon=\sqrt{\frac{m+1}{2}ln(\frac{1}{\delta})}$. We denote 
\begin{equation}	
	\Delta C^{-} = \mathbb{E}(\mu_{d_i^+}-2\sigma_{d_i^+})- \mathbb{E}(\mu_{\tilde{d}_j^{-}}-2\sigma_{\tilde{d}_j^{-}}) = \mathbb{E}(1 - (\mu_{\tilde{d}_j^{-}}-2\sigma_{\tilde{d}_j^{-}}))-\mathbb{E}(1-(\mu_{d_i^+}-2\sigma_{d_i^+})),
\end{equation}	
	where $\mathbb{E}(1 - (\mu_{\tilde{d}_j^{-}}-2\sigma_{\tilde{d}_j^{-}}))$ is the risk expectation of false negatives being labeled as \emph{matching} and $\mathbb{E}(1-(\mu_{d_i^+}-2\sigma_{d_i^+}))$ is the risk expectation of true positives being labeled as \emph{matching}.
	Based on the definition of VaR, we have $VaR^-(\tilde{d}_j^{-})=\mu_{\tilde{d}_j^{-}} + 2\sigma_{\tilde{d}_j^{-}}$, and $VaR^+(d_i^+)=1-(\mu_{d_i^+}-2\sigma_{d_i^+})$. 
	Denoting $\Delta VaR^{-} = VaR^-(\tilde{d}_j^{-}) - VaR^+(d_i^+)$, we have,
	\begin{equation}
	\begin{split}
		1 + \Delta VaR^{-} =& \mu_{\tilde{d}_j^{-}} + 2\sigma_{\tilde{d}_j^{-}} + \mu_{d_i^+}-2\sigma_{d_i^+} \\
		\leq & \mu_{\tilde{d}_j^{-}} + \mu_{\tilde{d}_j^{-}} + 2\varepsilon+\big{[}\mathbb{E}(\mu_{d_i^+}-2\sigma_{d_i^+})- \mathbb{E}(\mu_{\tilde{d}_j^{-}}-2\sigma_{\tilde{d}_j^{-}})\big{]} \\
		= & 2\mu_{\tilde{d}_j^{-}} + 2\varepsilon + \Delta C^{-}.
	\end{split}
	\end{equation} 
	
	
	Hence, for a randomly selected false negative and a randomly selected true positive, with probability at least $(1-\delta)^2$, we have 
	\begin{equation}
	\label{equation:fn-flip-prob}
		\mu_{\tilde{d}_j^{-}} \geq \frac{1}{2} + \frac{\Delta VaR^{-}}{2} - \sqrt{\frac{m+1}{2}ln(\frac{1}{\delta})} - \frac{\Delta C^{-}}{2}.
	\end{equation}
	Note that the probability of the above inequality does not hold is $[1-(1-\delta)^2]$. Suppose that there are totally $n$ true positives $d_i^+$ ranked after $\tilde{d}_j^{-}$ by LearnRisk such that each true positive, $d_i^+$, satisfies $\Delta VaR^{-} - \Delta C^{-} > \epsilon$. Then the probability of Inequality~\ref{equation:fn-flip-prob} fails can be approximated by $[1-(1-\delta)^2]^n$. That is, the probability of at least one of the true positives can support the Inequality~\ref{equation:fn-flip-prob} is $\{1 - [1-(1-\delta)^2]^n\}$. Let $1 - [1-(1-\delta)^2]^n=1-\delta'$, we can get $\delta=1-\sqrt{1-\delta'^{\frac{1}{n}}}$. Therefore, with probability at least $(1-\delta)$,
	\begin{equation}
		\mu_{\tilde{d}_j^{-}} \geq \frac{1}{2} + \frac{\epsilon}{2} - \sqrt{\frac{m+1}{2}ln[\frac{1}{1-(1-\delta^{\frac{1}{n}})^\frac{1}{2}}]}.
	\end{equation}
\end{proof}

Note that the total number of rule risk features ($m$) is usually limited (e.g., dozens or hundreds), while $n$ is usually much larger than $m$. By the exponential effect of $n$, the 3rd term on the right-hand side tends to become zero as the value of $n$ increases.

\subsection{Proof of Lemma~\ref{lemma:fn-deltac}}
\label{appendix:proof-fn-deltac}
\textbf{Lemma~\ref{lemma:fn-deltac}} \emph{
	\begin{equation}
	\begin{split}
	\Delta C^{-} \leq max\big{\{}\mathbb{E}(w_{d_i^{+}}(\hat{\mu}_{d_i^{+}}-2\hat{\sigma}_{d_i^+}))-
	\mathbb{E}(w_{\tilde{d}_j^{-}}(\hat{\mu}_{\tilde{d}_j^{-}}-2\hat{\sigma}_{\tilde{d}_j^{-}})), \mathbb{E}(w_{d_i^{+}}\hat{\mu}_{d_i^{+}}) -\mathbb{E}(w_{\tilde{d}_j^{-}}\hat{\mu}_{\tilde{d}_j^{-}})\big{\}},
	\end{split}
	\end{equation}
	where the $\hat{\mu}_{*}$ and $\hat{\sigma}_{*}$ denote the DNN output probability and its corresponding standard deviation respectively, $w_*$ denotes the learned weight of DNN risk feature.
}

\begin{proof}
For simplicity of presentation, let $N_{d_i^+}$ denote the weight normalization factor of $d_i^{+}$, or $N_{d_i^+}=\displaystyle\sum_{k=1}^{m}w_kZ_k +w_{d_i^+}$. Similarly, let $N_{\tilde{d}_j^{-}}$ denote the weight normalization factor of $\tilde{d}_j^{-}$, or $N_{\tilde{d}_j^{-}}=\displaystyle\sum_{k=1}^{m}w_kZ_k + w_{\tilde{d}_j^{-}}$.
According to the weight function defined by \emph{LearnRisk}~\citep{chen2019towards}, without loss of generality, we suppose that $w_{d_i^+}=w_{\tilde{d}_j^{-}}$. As a result, $N_{d_i^+}=N_{\tilde{d}_j^{-}}$. Based on Assumption~\ref{assumption:identical-risk-feature}, we have,
\begin{equation*}
\begin{split}
&\mathbb{E}(\mu_{d_i^+}-2\sigma_{d_i^+})- \mathbb{E}(\mu_{\tilde{d}_j^{-}}-2\sigma_{\tilde{d}_j^{-}}) \\
=&\mathbb{E}((\mu_{d_i^+}-2\sigma_{d_i^+})-(\mu_{\tilde{d}_j^{-}}-2\sigma_{\tilde{d}_j^{-}})) \\
=&\mathbb{E}\Bigg{(}\frac{1}{{N_{d_i^+}}}\Bigg{(}\displaystyle\sum_{k=1}^{m}w_k\mu_{f_k}Z_k +w_{d_i^+}\hat{\mu}_{d_i^+}-2\sqrt{\displaystyle\sum_{k=1}^{m}w_k^{2}\sigma^2_{f_k}Z_k +w^2_{d_i^+}\hat{\sigma}^2_{d_i^+}}\Bigg{)} - \\
&\frac{1}{{N_{\tilde{d}_j^{-}}}}\Bigg{(}\displaystyle\sum_{k=1}^{m}w_k\mu_{f_k}Z_k+w_{\tilde{d}_j^{-}}\hat{\mu}_{\tilde{d}_j^{-}}-2\sqrt{\displaystyle\sum_{k=1}^{m}w_k^{2}\sigma^2_{f_k}Z_k +w^2_{\tilde{d}_j^{-}}\hat{\sigma}^2_{\tilde{d}_j^{-}}}\Bigg{)}\Bigg{)}\\
 =&\mathbb{E}\Bigg{(}\frac{1}{N_{d_i^+}}\Bigg{(}\displaystyle\sum_{k=1}^{m}w_k\mu_{f_k}Z_k +w_{d_i^+}\hat{\mu}_{d_i^+}-2\sqrt{\displaystyle\sum_{k=1}^{m}w_k^{2}\sigma^2_{f_k}Z_k +w^2_{d_i^+}\hat{\sigma}^2_{d_i^+}}-\\
 & \displaystyle\sum_{k=1}^{m}w_k\mu_{f_k}Z_k-w_{\tilde{d}_j^{-}}\hat{\mu}_{\tilde{d}_j^{-}}+2\sqrt{\displaystyle\sum_{k=1}^{m}w_k^{2}\sigma^2_{f_k}Z_k +w^2_{\tilde{d}_j^{-}}\hat{\sigma}^2_{\tilde{d}_j^{-}}}\Bigg{)}\Bigg{)}\\
=&\mathbb{E}\Bigg{(}\frac{1}{N_{d_i^+}}\Bigg{(}w_{d_i^+}\hat{\mu}_{d_i^+}- w_{\tilde{d}_j^{-}}\hat{\mu}_{\tilde{d}_j^{-}}+ 2\bigg{[}\sqrt{\displaystyle\sum_{k=1}^{m}w_k^{2}\sigma^2_{f_k}Z_k + w^2_{\tilde{d}_j^{-}}\hat{\sigma}^2_{\tilde{d}_j^{-}}} - \sqrt{\displaystyle\sum_{k=1}^{m}w_k^{2}\sigma^2_{f_k}Z_k +w^2_{d_i^+}\hat{\sigma}^2_{d_i^+}} \bigg{]} \Bigg{)}\Bigg{)} \quad\cdots(S_1)\\
\end{split}
\end{equation*}
If $w_{\tilde{d}_j^{-}}\hat{\sigma}_{\tilde{d}_j^{-}} < w_{d_i^+}\hat{\sigma}_{d_i^+}$, then
\begin{equation*}
S_1 \leq \mathbb{E}\big{(}\frac{1}{N_{d_i^+}}\big{(}w_{d_i^+}\hat{\mu}_{d_i^+}- w_{\tilde{d}_j^{-}}\hat{\mu}_{\tilde{d}_j^{-}} \big{)}\big{)}.
\end{equation*}
If  $w_{\tilde{d}_j^{-}}\hat{\sigma}_{\tilde{d}_j^{-}} \geq w_{d_i^+}\hat{\sigma}_{d_i^+}$, then
\begin{equation*}
\begin{split}
S_1 &\leq \mathbb{E}\Bigg{(}\frac{1}{N_{d_i^+}}\Bigg{(}w_{d_i^+}\hat{\mu}_{d_i^+}- w_{\tilde{d}_j^{-}}\hat{\mu}_{\tilde{d}_j^{-}}+
2\bigg{[}\sqrt{w^2_{\tilde{d}_j^{-}}\hat{\sigma}^2_{\tilde{d}_j^{-}}} - \sqrt{w^2_{d_i^+}\hat{\sigma}^2_{d_i^+}} \bigg{]} \Bigg{)}\Bigg{)} \quad\cdots(S_2) \\
&=\mathbb{E}\big{(}\frac{w_{d_i^+}}{N_{d_i^+}}\big{(}\hat{\mu}_{d_i^+} -2\hat{\sigma}_{d_i^+}\big{)}\big{)} - \mathbb{E}\big{(}\frac{w_{\tilde{d}_j^{-}}}{N_{d_i^+}}\big{(}\hat{\mu}_{\tilde{d}_j^{-}} -2\hat{\sigma}_{\tilde{d}_j^{-}}\big{)}\big{)}
\end{split}
\end{equation*}
From step $S_1$ to step $S_2$, we apply the rule that if $a\geq0, b\geq0, c\geq0$ and $b\geq c$, then $\sqrt{a+b}-\sqrt{a+c} \leq \sqrt{b} - \sqrt{c}$. For simplicity of presentation, we denote the normalization of $\frac{w_{d_i^+}}{N_{d_i^+}}$ by $w_{d_i^+}$, and similarly, the normalized $w_{\tilde{d}_j^{-}}$.

Hence, we have
\begin{equation*}
\Delta C^{-} \leq max\big{\{}\mathbb{E}(w_{d_i^{+}}(\hat{\mu}_{d_i^{+}}-2\hat{\sigma}_{d_i^+}))-
\mathbb{E}(w_{\tilde{d}_j^{-}}(\hat{\mu}_{\tilde{d}_j^{-}}-2\hat{\sigma}_{\tilde{d}_j^{-}})), \mathbb{E}(w_{d_i^{+}}\hat{\mu}_{d_i^{+}}) -\mathbb{E}(w_{\tilde{d}_j^{-}}\hat{\mu}_{\tilde{d}_j^{-}})\big{\}},
\end{equation*}
where the $\hat{\mu}_{*}$ and $\hat{\sigma}_{*}$ denote the DNN output probability and its corresponding standard deviation respectively, $w_*$ denotes the learned weight of DNN risk feature.
\end{proof}

%
%

\subsection{Proof of Theorem~\ref{theorem:fp}}
\label{appendix:proof-fp}

Similarly, based on Assumption~\ref{assumption:identical-risk-feature}, we theoretically analyze the chance of a false positive being flipped from \emph{matching} to \emph{unmatching}. We first prove a lemma, and then prove Theorem~\ref{theorem:fp} based on the lemma.
%
%
\begin{lemma}
\label{lemma:fp}
	For a randomly selected pair $d_i^-$ from true negatives, we denote the mean of its equivalence probability by $\mu_{d_i^-}$, and the corresponding standard deviation by $\sigma_{d_i^-}$. For any $\delta \in (0, 1)$, with probability at least $(1-\delta)$, the following inequality holds
	\begin{equation*}
	\mathbb{E}(\mu_{d_i^-}+2\sigma_{d_i^-}) - (\mu_{d_i^-}+2\sigma_{d_i^-}) \leq \varepsilon,
	\end{equation*}
	where $\varepsilon=\sqrt{\frac{m+1}{2}ln(\frac{1}{\delta})}$, $m$ denotes the total number of rule risk features. Similarly, for a randomly selected false positive $\tilde{d}_j^{+}$ with the equivalence probability mean of $\mu_{\tilde{d}_j^{+}}$ and the standard deviation of $\sigma_{\tilde{d}_j^{+}}$, with probability at least $(1-\delta)$, the following inequality holds
	\begin{equation*}
	(\mu_{\tilde{d}_j^{+}}+2\sigma_{\tilde{d}_j^{+}}) - \mathbb{E}(\mu_{\tilde{d}_j^{+}}+2\sigma_{\tilde{d}_j^{+}}) \leq \varepsilon.
	\end{equation*}
\end{lemma}

\begin{proof}
	Consider a randomly selected pair $d_i^-$ from true negatives. The mean of its equivalence probability can be represented by
	\begin{equation*}
	\mu_{d_i^-}=\frac{\displaystyle\sum_{k=1}^{m}w_k\mu_{f_k}Z_k + w_{d_i^-}\hat{\mu}_{d_i^-}}{\displaystyle\sum_{k=1}^{m}w_kZ_k + w_{d_i^-}},
	\end{equation*}
	where $m$ denotes the number of rule risk features, $w_k$ denotes the weight of a risk feature $f_k$, $\mu_{f_k}$ is the equivalence probability mean of the feature $f_k$, $Z_k$ is a random variable indicates if a selected pair has this feature, and $\hat{\mu}_{d_i^-}$ is the output probability by a classifier with its weight $w_{d_i^-}$. Note that for a randomly selected true negative, the values of $w_k$ and $\mu_{f_k}$ for each risk feature are fixed, while $Z_k$, $\hat{\mu}_{d_i^-}$ are random variables. Note that the value of $w_{d_i^-}$ totally depends $\hat{\mu}_{d_i^-}$. 
	
  The standard deviation of its equivalence probability can also be represented by
	\begin{equation*}
	\sigma_{d_i^-}=\frac{1}{\displaystyle\sum_{k=1}^{m}w_kZ_k + w_{d_i^-}}\sqrt{\displaystyle\sum_{k=1}^{m}w_k^{2}\sigma_{f_k}^{2}Z_k+w^2_{d_i^-}\hat{\sigma}^2_{d_i^-}},
	\end{equation*}
	where $\hat{\sigma}^2_{d_i^-}$ denotes the corresponding variance of a classifier's output $\hat{\mu}_{d_i^-}$.
	
	Let $f(Z_1,...,Z_m, \hat{\mu}_{d_i^-})=\mu_{d_i^-}+2\sigma_{d_i^-}$. Now we proceed to consider the bounded differences property of $f$. As in the proof of lemma 1, for all $\mu\geq0, \sigma\geq0$, we have $0\leq\mu\pm 2\sigma\leq1$. Hence, by changing the value of $Z_k$, we have
	\begin{equation*}
	sup|f(z_1,...,z_k,...,z_m, \hat{\mu}_{d_i^-})-f(z_1,...,z'_k,...,z_m, \hat{\mu}_{d_i^-})| \leq 1
	\end{equation*}
	
	Similarly, the upper bound of $f$ by changing the value of $\hat{\mu}_{d_i^-}$ is,
	\begin{equation*}
	sup|f(z_1,...,z_m, \hat{\mu}_{d_i^-})-f(z_1,...,z_m, \hat{\mu}'_{d_i^-})| \leq 1
	\end{equation*}
	
	At this point, we have obtained the upper bounds of function $f(Z_1,...,Z_m, \hat{\mu}_{d_i^-})$ by changing any one of the variables. Denoting these bounds by $c_1, ..., c_m, c_{m+1}$, where $c_k=1, 1\leq k \leq m+1$, we have
	\begin{equation*}
	\displaystyle\sum_{k=1}^{m+1}c_k^2 = m+1.
	\end{equation*}
	
	By applying the McDiarmid's inequality, for all $\varepsilon > 0$, we have
	\begin{equation*}
	\mathbb{P}(\mathbb{E}(\mu_{d_i^-}+2\sigma_{d_i^-}) - (\mu_{d_i^-}+2\sigma_{d_i^-}) \geq \varepsilon)\leq exp(-\frac{2\varepsilon^2}{\sum_{k=1}^{m+1}c_k^{2}}) = exp(-\frac{2\varepsilon^2}{m+1}).
	\end{equation*}
	
	Let $\delta=exp(-\frac{2\varepsilon^2}{m+1})$, we can get that $\varepsilon=\sqrt{\frac{m+1}{2}ln(\frac{1}{\delta})}$. Hence, with the probability at least $1-\delta$, the inequality $\mathbb{E}(\mu_{d_i^-}+2\sigma_{d_i^-}) - (\mu_{d_i^-}+2\sigma_{d_i^-}) \leq \varepsilon$ holds.
	
	Similarly,  with the probability at least $1-\delta$, we have $(\mu_{\tilde{d}_j^{+}}+2\sigma_{\tilde{d}_j^{+}}) - \mathbb{E}(\mu_{\tilde{d}_j^{+}}+2\sigma_{\tilde{d}_j^{+}}) \leq \varepsilon$.
\end{proof}

%
%
	\noindent \textbf{Theorem~\ref{theorem:fp}} \emph{
	Given a false positive $\tilde{d}_j^{+}$, suppose that there are totally $n$ true negatives, denoted by $d_i^-$, ranked after $\tilde{d}_j^{+}$ by LearnRisk such that each true negative, $d_i^-$, satisfies 
	\begin{equation}	
	\Delta VaR^{+} - \Delta C^{+} > \epsilon, 
	\end{equation} 		
	in which $\Delta VaR^{+}=VaR^+(\tilde{d}_j^{+}) - VaR^-(d_i^{-})$, and $\Delta C^{+} = \mathbb{E}(\mu_{{\tilde{d}_j^{+}}}+2\sigma_{\tilde{d}_j^{+}})- \mathbb{E}(\mu_{d_i^{-}}+2\sigma_{d_i^{-}})$. Then, for any $\delta \in (0, 1)$, with probability at least $1-\delta$, its expectation of equivalence probability of $\tilde{d}_j^{+}$, $\mu_{\tilde{d}_j^{+}}$ estimated by \emph{LearnRisk}, satisfies 
	\begin{equation*}
	\mu_{\tilde{d}_j^{+}} \leq \frac{1}{2} - \frac{\epsilon}{2} + \sqrt{\frac{m+1}{2}ln[\frac{1}{1-(1-\delta^{\frac{1}{n}})^{\frac{1}{2}}}]},
	\end{equation*}
	in which $\mu_{*}$ denotes the mean of equivalence probability and $\sigma_{*}$ denotes its standard deviation.}

\begin{proof}
	With Lemma~\ref{lemma:fp}, with probability at least $(1-\delta)^2$, the following inequalities hold 
	\begin{equation*}
	\begin{split}
	&(\mu_{d_i^-}+2\sigma_{d_i^-})-\mathbb{E}(\mu_{d_i^-}+2\sigma_{d_i^-})+\mathbb{E}(\mu_{\tilde{d}_j^{+}}+2\sigma_{\tilde{d}_j^{+}}) - (\mu_{\tilde{d}_j^{+}}+2\sigma_{\tilde{d}_j^{+}}) \geq -2\varepsilon; \\
	&(\mu_{d_i^-}+2\sigma_{d_i^-}-\mu_{\tilde{d}_j^{+}}-2\sigma_{\tilde{d}_j^{+}})+\mathbb{E}(\mu_{\tilde{d}_j^{+}}+2\sigma_{\tilde{d}_j^{+}})-\mathbb{E}(\mu_{d_i^-}+2\sigma_{d_i^-}) \geq -2\varepsilon;
	\end{split}
	\end{equation*}
Hence, we have	
	\begin{equation}
	\label{equation:fp-inequality}
	\mu_{d_i^-}+2\sigma_{d_i^-}-\mu_{\tilde{d}_j^{+}}-2\sigma_{\tilde{d}_j^{+}} \geq -2\varepsilon-\big{[}\mathbb{E}(\mu_{\tilde{d}_j^{+}}+2\sigma_{\tilde{d}_j^{+}})- \mathbb{E}(\mu_{d_i^-}+2\sigma_{d_i^-})\big{]},
	\end{equation}
	where $\varepsilon=\sqrt{\frac{m+1}{2}ln(\frac{1}{\delta})}$. We denote 
\begin{equation}
\Delta C^{+} = \mathbb{E}(\mu_{\tilde{d}_j^{+}}+2\sigma_{\tilde{d}_j^{+}})- \mathbb{E}(\mu_{d_i^-}+2\sigma_{d_i^-}),
\end{equation}
 where $\mathbb{E}(\mu_{\tilde{d}_j^{+}}+2\sigma_{\tilde{d}_j^{+}})$ is the risk expectation of false positives being labeled as \emph{unmatching} and $\mathbb{E}(\mu_{d_i^-}+2\sigma_{d_i^-})$ is the risk expectation of true negatives being labeled as \emph{unmatching}.
	Based on the definition of VaR, we have $VaR^+(\tilde{d}_j^{+})=1-(\mu_{\tilde{d}_j^{+}} - 2\sigma_{\tilde{d}_j^{+}})$, and $VaR^-(d_i^-)=\mu_{d_i^-}+2\sigma_{d_i^-}$. Denoting $\Delta VaR^{+} = VaR^+(\tilde{d}_j^{+}) - VaR^-(d_i^-)$, we have
	\begin{equation}
	\label{equation:fp-deltavar}
	\begin{split}
	1 - \Delta VaR^{+} =& \mu_{d_i^-} + 2\sigma_{d_i^-} + \mu_{\tilde{d}_j^{+}}-2\sigma_{\tilde{d}_j^{+}} \\
	\geq & \mu_{\tilde{d}_j^{+}} + \mu_{\tilde{d}_j^{+}} - 2\varepsilon-\big{[}\mathbb{E}(\mu_{\tilde{d}_j^{+}}+2\sigma_{\tilde{d}_j^{+}})- \mathbb{E}(\mu_{d_i^-}+2\sigma_{d_i^-})\big{]} \\
	= & 2\mu_{\tilde{d}_j^{+}} - 2\varepsilon - \Delta C^{+}.
	\end{split}
	\end{equation} 
	In Equation~\ref{equation:fp-deltavar}, the inequality is obtained by applying the Inequality~\ref{equation:fp-inequality}. 
	Hence, for a randomly selected false positive and a randomly selected true negative, with probability at least $(1-\delta)^2$, the following inequality holds
	\begin{equation}
	\label{equation:fp-flip-prob}
	\mu_{\tilde{d}_j^{+}} \leq \frac{1}{2} - \frac{\Delta VaR^{+}}{2} + \sqrt{\frac{m+1}{2}ln(\frac{1}{\delta})} + \frac{\Delta C^{+}}{2}.
	\end{equation}
	
	Note that the probability of the above inequality does not hold is $[1-(1-\delta)^2]$. Suppose that there are totally $n$ true negatives, denoted by $d_i^-$, ranked after $\tilde{d}_j^{+}$ by LearnRisk such that each true negative, $d_i^-$, satisfies $\Delta VaR^{+} - \Delta C^{+} > \epsilon$. Then the probability of Inequality~\ref{equation:fp-flip-prob} fails can be approximated by $[1-(1-\delta)^2]^n$. That is, the probability of at least one of the true negatives can support the Inequality~\ref{equation:fp-flip-prob} is $\{1 - [1-(1-\delta)^2]^n\}$. Let $1 - [1-(1-\delta)^2]^n=1-\delta'$, we can get $\delta=1-\sqrt{1-\delta'^{\frac{1}{n}}}$. Therefore, with probability at least $(1-\delta)$, we have
	\begin{equation}
	\mu_{\tilde{d}_j^{+}} \leq \frac{1}{2} - \frac{\epsilon}{2} + \sqrt{\frac{m+1}{2}ln[\frac{1}{1-(1-\delta^{\frac{1}{n}})^{\frac{1}{2}}}]},
	\end{equation}
\end{proof}

\subsection{Proof of Lemma~\ref{lemma:fp-deltac}}
\label{appendix:proof-fp-deltac}

%
%

\textbf{Lemma~\ref{lemma:fp-deltac}} \emph{
	\begin{equation}
	\begin{split}
	\Delta C^{+} \leq max\big{\{}\mathbb{E}(w_{\tilde{d}_j^{+}}(\hat{\mu}_{\tilde{d}_j^{+}}+2\hat{\sigma}_{\tilde{d}_j^{+}}))-
	\mathbb{E}(w_{d_i^-}(\hat{\mu}_{d_i^-}+2\hat{\sigma}_{d_i^-})), \mathbb{E}(w_{\tilde{d}_j^{+}}\hat{\mu}_{\tilde{d}_j^{+}}) -\mathbb{E}(w_{d_i^-}\hat{\mu}_{d_i^-})\big{\}},
	\end{split}
	\end{equation}
	where the $\hat{\mu}_{*}$ and $\hat{\sigma}_{*}$ denote the DNN output probability and its corresponding standard deviation respectively, $w_*$ denotes the learned weight of DNN risk feature.}

\begin{proof}
For simplicity of presentation, let $N_{\tilde{d}_j^{+}}$ denote the weight normalization factor of $\tilde{d}_j^{+}$, $N_{\tilde{d}_j^{+}}=\displaystyle\sum_{k=1}^{m}w_k\mu_{f_k}Z_k+w_{\tilde{d}_j^{+}}$. Similarly, $N_{d_i^-}$ denote the weight normalization factor of $d_i^-$, $N_{d_i^-}=\displaystyle\sum_{k=1}^{m}w_k\mu_{f_k}Z_k+w_{d_i^-}$. As in the proof of Lemma 1, we suppose that $N_{\tilde{d}_j^{+}}=N_{d_i^-}$. Based on Assumption~\ref{assumption:identical-risk-feature}, we have,
\begin{equation*}
\begin{split}
&\mathbb{E}(\mu_{\tilde{d}_j^{+}}+2\sigma_{\tilde{d}_j^{+}})- \mathbb{E}(\mu_{d_i^-}+2\sigma_{d_i^-}) \\
=&\mathbb{E}((\mu_{\tilde{d}_j^{+}}+2\sigma_{\tilde{d}_j^{+}})-(\mu_{d_i^-}+2\sigma_{d_i^-})) \\
=&\mathbb{E}\Bigg{(}\frac{1}{N_{\tilde{d}_j^{+}}}\Bigg{(}\displaystyle\sum_{k=1}^{m}w_k\mu_{f_k}Z_k+w_{\tilde{d}_j^{+}}\hat{\mu}_{\tilde{d}_j^{+}}+2\sqrt{\displaystyle\sum_{k=1}^{m}w_k^{2}\sigma^2_{f_k}Z_k+w^2_{\tilde{d}_j^{+}}\hat{\sigma}^2_{\tilde{d}_j^{+}}}\Bigg{)} - \\
&\frac{1}{N_{d_i^-}}\Bigg{(}\displaystyle\sum_{k=1}^{m}w_k\mu_{f_k}Z_k+w_{d_i^-}\hat{\mu}_{d_i^-}+2\sqrt{\displaystyle\sum_{k=1}^{m}w_k^{2}\sigma^2_{f_k}Z_k+w^2_{d_i^-}\hat{\sigma}^2_{d_i^-}}\Bigg{)}\Bigg{)} \\
=&\mathbb{E}\Bigg{(}\frac{1}{N_{\tilde{d}_j^{+}}}\Bigg{(}\displaystyle\sum_{k=1}^{m}w_k\mu_{f_k}Z_k+w_{\tilde{d}_j^{+}}\hat{\mu}_{\tilde{d}_j^{+}}+2\sqrt{\displaystyle\sum_{k=1}^{m}w_k^{2}\sigma^2_{f_k}Z_k+w^2_{\tilde{d}_j^{+}}\hat{\sigma}^2_{\tilde{d}_j^{+}}} - \\ &\displaystyle\sum_{k=1}^{m}w_k\mu_{f_k}Z_k-w_{d_i^-}\hat{\mu}_{d_i^-}-2\sqrt{\displaystyle\sum_{k=1}^{m}w_k^{2}\sigma^2_{f_k}Z_k+w^2_{d_i^-}\hat{\sigma}^2_{d_i^-}}\Bigg{)}\Bigg{)} \\
=&\mathbb{E}\Bigg{(}\frac{1}{N_{\tilde{d}_j^{+}}}\Bigg{(}w_{\tilde{d}_j^{+}}\hat{\mu}_{\tilde{d}_j^{+}}-w_{d_i^-}\hat{\mu}_{d_i^-}+2\bigg{[}\sqrt{\displaystyle\sum_{k=1}^{m}w_k^{2}\sigma^2_{f_k}Z_k+w^2_{\tilde{d}_j^{+}}\hat{\sigma}^2_{\tilde{d}_j^{+}}} - \sqrt{\displaystyle\sum_{k=1}^{m}w_k^{2}\sigma^2_{f_k}Z_k+w^2_{d_i^-}\hat{\sigma}^2_{d_i^-}} \bigg{]} \Bigg{)}\Bigg{)} \quad\cdots(S_3)\\
\end{split}
\end{equation*}
If $w_{\tilde{d}_j^{+}}\hat{\sigma}_{\tilde{d}_j^{+}} < w_{d_i^-}\hat{\sigma}_{d_i^-}$, then
\begin{equation*}
S_3 \leq \mathbb{E}\big{(}\frac{1}{N_{\tilde{d}_j^{+}}}\big{(}w_{\tilde{d}_j^{+}}\hat{\mu}_{\tilde{d}_j^{+}}- w_{d_i^-}\hat{\mu}_{d_i^-} \big{)}\big{)}.
\end{equation*}
If  $w_{\tilde{d}_j^{+}}\hat{\sigma}_{\tilde{d}_j^{+}} \geq w_{d_i^-}\hat{\sigma}_{d_i^-}$, then
\begin{equation*}
\begin{split}
S_3 &\leq \mathbb{E}\Bigg{(}\frac{1}{N_{\tilde{d}_j^{+}}}\Bigg{(}w_{\tilde{d}_j^{+}}\hat{\mu}_{\tilde{d}_j^{+}}- w_{d_i^-}\hat{\mu}_{d_i^-}+
2\bigg{[}\sqrt{w^2_{\tilde{d}_j^{+}}\hat{\sigma}^2_{\tilde{d}_j^{+}}} - \sqrt{w^2_{d_i^-}\hat{\sigma}^2_{d_i^-}} \bigg{]} \Bigg{)} \Bigg{)} \quad\cdots(S_4) \\
&=\mathbb{E}\big{(}\frac{w_{\tilde{d}_j^{+}}}{N_{\tilde{d}_j^{+}}}\big{(}\hat{\mu}_{\tilde{d}_j^{+}} +2\hat{\sigma}_{\tilde{d}_j^{+}}\big{)}\big{)} - \mathbb{E}\big{(}\frac{w_{d_i^-}}{N_{\tilde{d}_j^{+}}}\big{(}\hat{\mu}_{d_i^-} +2\hat{\sigma}_{d_i^-})\big{)}\big{)}
\end{split}
\end{equation*}
From step $S_3$ to step $S_4$, we apply the rule that if $a\geq0, b\geq0, c\geq0$ and $b\geq c$, then $\sqrt{a+b}-\sqrt{a+c} \leq \sqrt{b} - \sqrt{c}$. For simplicity of presentation, we denote the normalization of $\frac{w_{\tilde{d}_j^{+}}}{N_{\tilde{d}_j^{+}}}$ by $w_{\tilde{d}_j^{+}}$, and similarly, the normalized $w_{d_i^-}$. 
Hence, we have
\begin{equation*}
\Delta C^{+} \leq max\big{\{}\mathbb{E}(w_{\tilde{d}_j^{+}}(\hat{\mu}_{\tilde{d}_j^{+}}+2\hat{\sigma}_{\tilde{d}_j^{+}}))-
\mathbb{E}(w_{d_i^-}(\hat{\mu}_{d_i^-}+2\hat{\sigma}_{d_i^-})), \mathbb{E}(w_{\tilde{d}_j^{+}}\hat{\mu}_{\tilde{d}_j^{+}}) -\mathbb{E}(w_{d_i^-}\hat{\mu}_{d_i^-})\big{\}},
\end{equation*}
where the $\hat{\mu}_{*}$ and $\hat{\sigma}_{*}$ denote the DNN output probability and its corresponding standard deviation respectively, $w_*$ denotes the learned weight of DNN risk feature.
\end{proof}

	\vskip 0.2in
	\bibliography{reference}

\end{document}